\documentclass[11pt]{article}
\usepackage[margin=1in]{geometry}
\usepackage{amsfonts}
\usepackage{amsmath,amsthm,amssymb}
\usepackage{latexsym}
\usepackage{epic}
\usepackage{epsfig}
\usepackage[colorlinks=true, allcolors=magenta]{hyperref}
\usepackage{verbatim}
\usepackage[justification=centering]{caption}
\usepackage{enumitem}
\usepackage[dvipsnames]{xcolor}
\allowdisplaybreaks[1]
\usepackage[style=alphabetic,natbib=true]{biblatex}
\usepackage{algorithm}
\usepackage[most]{tcolorbox}
\tcbset{rounded corners}
\usepackage{setspace}
\usepackage[noend]{algpseudocode}
\algrenewcommand\algorithmicrequire{\textbf{Input:}}
\algrenewcommand\algorithmicensure{\textbf{Output:}}

\usepackage[normalem]{ulem}

\usepackage{ifthen}
\newboolean{coltsubmission}
\setboolean{coltsubmission}{false}

\newtheorem{theorem}{Theorem}[section]
\newtheorem{corollary}[theorem]{Corollary}
\newtheorem{lemma}[theorem]{Lemma}

\newtheorem{definition}[theorem]{Definition}
\newtheorem{remark}[theorem]{Remark}

\newtheorem{fact}[theorem]{Fact}

\newtheorem{open}[theorem]{Open Question}

\newenvironment{alg}{\begin{algorithm}\begin{onehalfspace}\begin{algorithmic}[1]}{\end{algorithmic}\end{onehalfspace}\end{algorithm}}

\newcommand{\floor}[1]{\lfloor {#1} \rfloor}

\renewcommand{\hat}{\widehat}
\renewcommand{\tilde}{\widetilde}

\def\min{\qopname\relax n{min}}
\def\max{\qopname\relax n{max}}

\newcommand{\expect}[1]{%
  \mathop{\qopname\relax n{\mathbb{E}}}\!\left[#1\right]
}
\newcommand{\expected}[1]{%
  \qopname\relax n{\mathbb{E}}_{#1}
}
\newcommand{\expects}[2]{%
  \qopname\relax n{\mathbb{E}}_{#1}\!\left[#2\right]
}

\newcommand{\prob}[1]{%
  \mathop{\qopname\relax n{\mathbb{P}}}\!\left[#1\right]
}

\newcommand{\probs}[2]{%
  \qopname\relax n{\mathbb{P}}_{#1}\!\left[#2\right]
}

\newcommand{\NN}{\mathbb{N}}

\newcommand{\HH}{\mathbb{H}}
\newcommand{\DD}{\mathbb{D}}

\def\A{\mathcal{A}}
\def\B{\mathcal{B}}
\def\C{\mathcal{C}}
\def\D{\mathcal{D}}

\def\H{\mathcal{H}}

\def\Q{\mathcal{Q}}

\def\S{\mathcal{S}}

\def\U{\mathcal{U}}

\def\X{\mathcal{X}}

\def\eps{\epsilon}
\def\sse{\subseteq}

\newcommand{\eat}[1]{}

\newenvironment{lp*}{\begin{equation*}  \begin{array}{lll}}{\end{array}\end{equation*}}

\usepackage[suppress]{color-edits}
\addauthor[Alireza]{az}{OliveGreen}
\addauthor[Shaddin]{sd}{blue}
\newcommand{\unsddelete}[1]{#1}
\newcommand{\unsdreplace}[2]{#1}

\newcommand{\maj}{\mathsf{Maj}}
\newcommand{\erm}{\mathsf{ERM}}
\newcommand{\oig}{\mathsf{OIG}}

\newcommand{\simnr}{\sim_\mathsf{nr}}

\newcommand{\dom}{\mathsf{dom}}
\newcommand{\trans}{\mathsf{Trans}}
\newcommand{\dtv}{\mathsf{d_{TV}}}
\newcommand{\supp}{\mathsf{supp}}

\newcommand{\tunif}{\mathsf{MTestUnif}}
\newcommand{\tunifstandard}{\mathsf{TestUnif}}
\newcommand{\mtest}{\mathsf{Test}}

\DeclareMathAlphabet{\bbold}{U}{bbold}{m}{n}
\newcommand{\id}{\ensuremath{\bbold{1}}}

\newcounter{alg}
\renewcommand{\thealg}{\arabic{alg}}

\newtcolorbox[auto counter]{algoboxinner}[1][]{
  colback=white,
  colframe=black,
  rounded corners,
  boxrule=0.8pt,
  arc=8pt,
  left=6pt,right=6pt,top=6pt,bottom=6pt,
  before upper={
    \refstepcounter{alg}%
    \textbf{Algorithm~\thealg.}~\textbf{#1}
  }
}

\newenvironment{algobox}[1][tpb]{%
  \begin{figure}[#1]
  \begin{algoboxinner}
}{%
  \end{algoboxinner}
  \end{figure}
}

\title{Relatively Smart: \\ A New Approach for Instance-Optimal Learning}

\author{
 Shaddin Dughmi\thanks{shaddin@usc.edu. Supported by NSF Grant CCF-2432219. Part of this work was done while the author was on sabbatical as the Carter and Tania Neild visiting professor at Northwestern University, as well as a visiting professor in the Data Science Institute at the University of Chicago.}\\
 University of Southern California
 \and 
 Alireza F. Pour\thanks{alireza.fathollahpour@uwaterloo.ca. Supported by a David Cheriton Scholarship and a Vector Institute Research Grant. }\\
 University of Waterloo}

\date{}

\begin{document}

\maketitle
\begin{abstract}
    We revisit the framework of \emph{Smart PAC learning}, which seeks supervised learners which compete with semi-supervised learners that are provided full knowledge of the \emph{marginal} distribution on unlabeled data. Prior work has shown that such marginal-by-marginal guarantees are possible for ``most'' marginals, with respect to an arbitrary fixed and known measure, but not more generally. We discover that this failure can be attributed to an ``indistinguishability'' phenomenon: There are marginals which cannot be statistically distinguished from other marginals that require different learning approaches.        
    In such settings, semi-supervised learning cannot certify its guarantees from unlabeled data, rendering them arguably non-actionable.

    We propose \emph{relatively smart learning}, a new framework which  demands that a supervised learner compete only with the best ``certifiable'' semi-supervised guarantee. We show that such modest relaxation suffices to bypass the impossibility results from prior work. In the distribution-free setting, we show that the One-Inclusion Graph learner is relatively smart up to squaring the sample complexity,  and show that no supervised learning algorithm can do better. For distribution-family settings, we show that relatively smart learning can be impossible or can require idiosyncratic learning approaches, and its difficulty can be non-monotone in the inclusion order on distribution families. 
    


\end{abstract}



\section{Introduction}
\label{sec:intro}
A substantial portion of research in learning theory proceeds within the worst-case  tradition characteristic of theoretical computer science more broadly. This includes\unsddelete{, most prominently,} the PAC model \citep{valiant_theory_1984} and its numerous extensions, which typically impose a “flat” inductive bias on the hypothesis class and/or data-generating distribution, then leave the choice of instance to an adversary. \unsddelete{Performance guarantees are consequently evaluated in the worst case over all hypotheses and distributions consistent with these assumptions.} 

It has been argued that this perspective is somewhat removed from practical machine learning, which tends to be more adaptive to the specifics of its deployment domain. Techniques such as  unsupervised pretraining, hyperparameter optimization, post-hoc model refinement, and incorporation of domain expertise are often used to tailor the learner to the data distribution. Partly as a result of these adaptation mechanisms, the \unsddelete{observed} performance of real-world machine learning systems is typically far from worst-case \citep{chapelle_semi-supervised_2006, erhan_why_2010,guo2017calibration,zhang2017understanding,frankle2018the,devlin2019bert,chen2020simple,yang2020hyperparameter}. 

In this paper, we investigate one facet of this divide: how a learner can be tailored to the unlabeled data distribution, \unsdreplace{also known as }{a.k.a.~}the \emph{marginal distribution}. We focus on statistical aspects of this question, as captured by the sample complexity (or equivalently, the error~rate)\unsddelete{ of learning}. We restrict attention to what is arguably the most historically instructive learning setting: realizable binary classification. 

\subsection*{Distribution-Fixed and Semi-Supervised Learning}
The original PAC learning model considers \emph{fully supervised learners} which only receive labeled samples from the distribution. A particularly powerful extension of this paradigm additionally provides the learner with full knowledge of the marginal distribution of unlabeled data.  This \emph{distribution-fixed} model of learning was studied by  \citet{benedek_learnability_1991}, where they characterize learnability qualitatively for each marginal in terms of the existence of finite covers of the hypothesis class, at every error scale, with respect to the disagreement metric induced by the marginal. They also derive (not quite tight) upper and lower bounds on the sample complexity in terms of these covers. Formal evidence that knowing the marginal can enable learning was subsequently provided by~\citet{dudley_metric_1994}. 

As pointed out by \citet{ben2008does}, distribution-fixed learning can be viewed as a utopian idealization of \emph{semi-supervised learning}---the paradigm which augments labeled data with more plentiful unlabeled data. There is a long line of work which explores more realistic formulations of semi-supervised learning, seeking to understand whether, when, and how much finite unlabeled data helps in learning   \citep[e.g.][]{ben2008does,balcan2010discriminative,darnstadt_unlabeled_2013,globerson_effective_2017, gopfert_when_2019,pmlr-v97-golovnev19a,pukdee2023learning}. There are many shades of these questions explored by this body of work to which we cannot do justice, but for our purposes the gist is as follows: Unlabeled data does not improve minimax error rates in distribution-free settings,\footnote{Recall that the lower-bound in the fundamental theorem of PAC learning is robust to knowledge of the marginal.
}
 but can lead to drastic marginal-by-marginal improvements, as well as minimax improvements for some distribution-family settings and under ``compatibility'' assumptions between marginals and hypotheses. 

\subsection*{Smart Learning}
Instead of further exploring the power and limits of semi-supervised learning, we \unsddelete{take a different tack in this paper. We} build on the closely-related framework of \emph{Smart Learning} introduced by \citet{darnstadt2011smart_journal}. Roughly speaking, a smart learner is a fully-supervised learner  which \emph{does about as well as if it knew the marginal distribution already, even though it doesn't.} This is an \unsddelete{(approximate)} instance-optimality guarantee with respect to marginals: a smart learner approximately matches the optimal distribution-fixed error rate (or equivalently, sample complexity) for every marginal simultaneously. 

While ambitious, the goal of smart learning seems plausible at first glance: By eschewing minimax guarantees across marginals, the learner is permitted to perform poorly for ``hard'' marginals such as those appearing in the fundamental theorem of PAC learning, while paying special attention to those marginals most amenable to semi-supervised learning techniques.
 Indeed, \citet{darnstadt2011smart_journal} show a compelling, albeit qualified, positive result via an innovative application of the minimax theorem for zero-sum games: Smart learning is possible in general distribution-free settings for ``most'' marginals, where ``most'' is quantified with respect a prior distribution on marginals that is given in advance. Unfortunately, any hope of removing this qualification was dashed by subsequent work of \citet{darnstadt_unlabeled_2013}, strengthening an earlier result of \citet{dudley_metric_1994}:  They  exhibit a hypothesis class and family of marginals with distribution-fixed  error rates rapidly and uniformly tending to zero, whereas for any fully-supervised learner---not equipped with foreknowledge of the marginal---and any finite sample size there is an instance where the learner performs essentially no better than random guessing. 

\subsection*{From Smart to Relatively Smart Learning}
The \unsdreplace{results of \citep{darnstadt2011smart_journal,darnstadt_unlabeled_2013}}{aforementioned results} might appear to close the book on smart learning for general hypothesis classes and sufficiently rich distribution families. This, however, is where our work comes in. Our contribution emanates from the following realization: Smart learning fails when a distribution-fixed learner $\A_\D$ catered to a marginal $\D$ cannot use its unlabeled data to distinguish $\D$ from other marginals $\D'$ where $\A_\D$ performs much worse. In other words, it is impossible to detect misspecifications of the marginal that are consequential to learning by merely inspecting the unlabeled data, rendering error guarantees impossible to certify  prior to procuring labels.  Our results will imply that this is the only qualitative obstacle to smart learning.\footnote{Smart learning is self-evidently an unsupervised learning task: that of learning characteristics of an unknown $\D$ that are most pertinent for subsequent supervised learning.  Our stated  obstacle \unsdreplace{concerns a  prima facie easier task more akin to testing: $\D$ is fixed and must merely be distinguished from other distributions $\D'$ which prohibit similar supervised learning approaches. Perhaps surprisingly, our results reveal an equivalence between learning and testing which holds here, but fails more generally in statistics (see e.g.~\cite{batu2000testing}).}{is more akin to testing, which in general can be strictly easier than learning \citep[e.g.][]{batu2000testing}. Perhaps surprisingly, our results imply equivalence here.} } \looseness=-1 

Sparked by this realization we introduce \emph{relatively smart learning}, which minimally relaxes smart learning to ``price-in'' the above-described obstacle for each marginal distribution $\D$. Informally, for each marginal $\D$ and learner $\A=\A_\D$ we seek to compete not with the error $\A$ incurs on  $\D$ (as in smart learning), but rather with the best \emph{certifiable upperbound} on that error that can be calculated from the unlabeled data. By this we mean that there is a real-valued \emph{certifier} $\C$ which estimates $\A$'s error from the unlabeled data, and we require $\C$ to  be \emph{sound} in the following strong sense: its error estimate must in-expectation upper bound $\A$'s error for all admissible instances. Importantly, even if $\A=\A_\D$ is tailored to a particular marginal $\D$, $\C$ must soundly certify $\A$'s error for all admissible  marginals~$\D'$, even if $\D' \neq \D$. This requirement  serves to relax distribution-fixed error rates upwards by effectively taking the worst case over all $\D'$ that are indistinguishable from~$\D$, making the design of instance-wise competitive learners more achievable.
A fully-supervised learner is now said to be \emph{relatively smart}\footnote{Read: Smart relative to every certifiable error guarantee.} if it (approximately) matches the best certifiable error rate for every admissible marginal distribution. 

\subsection*{Our Results}
We begin in the distribution-free setting with general hypothesis classes. Our main positive result (Theorem~\ref{thm:oig_square_smart}) is that relatively smart learning is possible at the cost of a quadratic blowup in the number of samples and constant blowup in the error, as compared to the best certifiable distribution-dependent error rates. In particular, this is achieved by the familiar One-Inclusion-Graph (OIG) learner of \citet{haussler_predicting_1994}.  Our main negative result (Theorem~\ref{thm:no_learner_m2_smart}) is that this is essentially tight, in that every relatively smart learner must suffer a near-quadratic blowup in sample complexity. Since the latter result is quite technical, we also provide a simpler proof of the same bound for the special case of the OIG and Empirical Risk Minimization (ERM) learners (Theorem~\ref{thm:oig_erm_fail}). We also discuss the intriguing question of whether ERM, or some other ``simple'' and typically-tractable learner, is relatively smart (Open Question~\ref{open:erm}).

We then examine distribution-family settings. We observe in Corollary~\ref{cor:oig_smart_family} that our main positive result (Theorem~\ref{thm:oig_square_smart}) extends to families characterized only by the allowable subsets of the domain on which data must be supported (e.g. manifolds satisfying some algebraic or topological requirements). Beyond such ``simple'' families, we show that relatively smart learning starts to exhibit  richer and more nuanced structure: There are families where relatively smart learning is completely impossible (Theorem~\ref{thm:dfamily_nosmart}), and others where it is possible but not by straightforward approaches such as OIG or ERM  (Theorem~\ref{thm:dfamily_oig_erm_notsmart}). On a more meta level, we show in Corollary~\ref{cor:nonmonotone} that the difficulty of relatively smart learning, unlike traditional PAC learning or Smart learning, can be non-monotone in the inclusion order on distribution families. We attribute this to the shifting benchmark of certifiable error rates, where the soundness requirement introduces dependence on the family as a whole.

We note that our negative results hold for countable domains, and our positive results hold more generally for domains, hypothesis classes, and distributions jointly satisfying standard measurability assumptions \citep[see e.g.][]{shalev-shwartz_understanding_2014}. 
As is common in learning theory, we take a hands-off approach to the measure-theoretic details. 

\subsection*{Connection to Testable Learning}
We would be remiss not to discuss connections between relatively smart learning and the framework of \emph{testable learning}, originally introduced by~\citet{rubinfeld2023testing} and spawning a rapid succession of followup work since \citep{gollakota2023moment,gollakota2023tester,  diakonikolas2023efficient,klivans2024testable,gollakota2024efficient}.
Our certifiers can be viewed as real-valued analogues to the testers from that framework, where soundness in our framework is analogous to their requirement that the learner performs well for every distribution which passes the test. Whereas testable learning is concerned with the design of learner/tester pairs for a specific distribution or distributional property, we instead use these objects as our \emph{benchmark} for \emph{every distribution separately}. This is the essence of the connection, as well as the main difference. 

There are other important differences: (a)  \unsdreplace{The literature on testable}{Testable} learning is primarily concerned with computational complexity, with the notable exception of~\cite{gollakota2023moment}. (b)~\unsdreplace{Testers in that framework}{Their testers} are permitted to use labeled data,  though this happens to not be necessary for many of the problems considered. (c) Their focus is on agnostic learning, with the analogous realizable question rendered trivial by (b); a notable exception is in the context of distribution shift~\citep{klivans2024testable}. 

\subsection*{Roadmap}
Section~\ref{sec:relatively_smart_learning} introduces relatively smart learning and discusses its basic properties. Section~\ref{sec:dfree_oig_erm} examines the familiar ERM and OIG learners in the context of  distribution-free relatively smart learning, and presents our main positive result for OIG as well as a complementary negative result for both learners. Section~\ref{sec:dfree_lb} provides a tight negative result for distribution-free relatively smart learning which holds for all learners.
Section~\ref{sec:dfamily} explores relatively-smart learning in distribution-family settings, outlining similarities and differences from the distribution-free setting.

\subsection*{Basic Notation}
For a probability distribution $\D$ on some domain $\X$ and an event $Y \sse \X$, we use $\D[Y]=\probs{\D}{Y}$ 
as shorthand for the probability of~$Y$, and $\D_{|Y} = \probs{\D}{\ .\ | Y}$ as shorthand for the conditional distribution of $\D$ given~$Y$. When the domain $\X$ is countable we use $\D[x]=\probs{\D}{x}$ to denote the probability of $x\in \X$, and use $\supp(\D) = \{x \in \X: \D[x] > 0\}$ to denote the support of $\D$.  For a finite multiset $S$ we use $\D_S$ to denote the uniform distribution on~$S$. Given a function $h$ defined on some domain $\X$, we use $h|Y$ to denote its restriction to some $Y \sse \X$. For a set $\X$ we use $\X^*$ to denote the family of finite sequences over $\X$. For a predicate or probability event $E$ we use $\id[E] \in \{0,1\}$ to denote the indicator of $E$. Finally, we use standard order-of-growth (big-Oh) notation, though for functions $f(\eta,m)$ and $g(m)$  we write $f= O_\eta(g)$ to indicate that $f=O(g)$ whenever $\eta$ is a fixed constant. \looseness=-1 

\section{Relatively Smart Learning}\label{sec:relatively_smart_learning}

We begin in the standard PAC learning setting of realizable binary classification. There is a \emph{data domain} $\X$ and a \emph{hypothesis class} $\H \sse \{0,1\}^\X$. Unlabeled data comes from a \emph{marginal distribution} $\D$ on $\X$, and is labeled with some \emph{ground truth} hypothesis $h \in \H$.  We use $\D_h$ to denote the joint distribution of labeled data $(x,y)$ with $x \sim \D$ and $y=h(x)$.   In the \emph{distribution-free setting} $\D$ can be arbitrary; more generally, we also consider settings where $\D$ is restricted to some \emph{distribution-family}~$\DD$.
For a \emph{predictor} $f: \X \to \{0,1\}$ and labeled data distribution $\D_h$, we denote the \emph{expected 0-1 loss}, or simply \emph{loss} or \emph{error},  of $f$ on $\D_h$ by $L(f,D_h) = \expects{(x,y) \sim \D_h}{\id[f(x) \neq y]}$. For a finite multiset $T$ of labeled data we overload notation and denote $L(f,T) = L(f,\D_T) = \frac{1}{|T|} \sum_{(x,y) \in T} \id[f(x) \neq y]$. 

A \emph{(fully-supervised) learner} $\A$ takes as input a sequence $S= (x_1,y_1), \ldots, (x_m,y_m)$ of labeled samples---often referred to as \emph{training data}---drawn i.i.d.~from $\D_h$, and outputs a predictor $\A(S): \X \to \{0,1\}$. We depart from the usual PAC approach by evaluating a learner by its expected error rather than its high probability error; this is merely for expository convenience, as all our results translate to PAC guarantees by standard arguments. More importantly, we measure a learner's error as a function of the marginal~$\D$, which permits comparing learners catered to the marginal---we call these \emph{$\D$-fixed learners}---to learners which are provided no such knowledge. We therefore define distribution-dependent error rates as follows. \looseness=-1

\begin{definition}\label{def:err_rate}
  A \emph{distribution-dependent error rate} is a function $\epsilon: \DD \times \NN \to [0,1]$, where $\epsilon(\D,m)$ is an error associated with a distribution $\D \in \DD$ and a number of samples $m$. 
\end{definition}
\begin{definition}\label{def:error_learner}
  For a learner $\A$, let its distribution-dependent error rate $\epsilon_\A(\D,m)$ be the worst-case, over hypotheses $h \in \H$, of its expected error with respect to $\D_h$ when given $m$ i.i.d.~samples from $\D_h$, i.e., \(\epsilon_\A(\D,m) = \sup_{h \in \H}  \expects{S \sim \D_h^m}{L\left(\A(S),\D_h\right)}.\) 
  \end{definition}

 We pair learners with functions that \emph{certify} their distribution-dependent errors from unlabeled data. We refer to those as \emph{certifiers}. We require certifiers to be \emph{sound}, meaning that they never underestimate the learner's error for any distribution in the class. 
\begin{definition}[Sound certifier]
Let $\A$ be a learner. We say a function $\C: \X^* \to [0,1]$ is a \emph{sound certifier} for $\A$ if for every $\D \in \DD$, $m \in \NN$, and $S \sim \D^m$ we have $\expects{}{\C(S)} \geq \epsilon_\A(\D,m)$.
\end{definition} 
Whereas smart PAC learning seeks to compete with the best distribution-fixed learner for each marginal, we instead propose a more modest benchmark: We only credit a distribution-fixed learner with error rates witnessed by a sound certifier.  This gives rise to \emph{certifiable error rates}.

\begin{definition}[Certifiable error rate]\label{def:certifiable_error_rate}
  A distribution-dependent error rate $\epsilon(.,.)$ is \emph{certifiable} if for each $\D \in \DD$, there exists a $\D$-fixed learner $\A$ and a sound certifier $\C$ for $\A$ such that for each $m \in \NN$ and $S \sim \D^m$, we have $\expects{}{\C(S)} \leq \epsilon(\D,m)$.
\end{definition}

Note that in the above definition, we also have $\epsilon_\A(\D,m) \leq \epsilon(\D,m)$ for $\A=\A_\D$ by soundness of $\C$ for $\A$. More importantly, the bite in this definition comes from the fact that we require $\C$ to be sound for $\A$ everywhere (i.e., for all $\D' \in \DD$), even though $\A$ is catered to $\D$ specifically. This is what effectively forces us to take the worst case error rate over all $\D'$ that are indistinguishable from $\D$ with $m$ samples. It is also important to note that our definition allows the certifiable error rate to be achieved by different learners $\A_{\D}$ that are fixed for each distribution $\D \in \DD$. The only requirement is that the learner has a certifier that can soundly witness its purported error rate for \emph{all} distributions in $\DD$, even those distributions on which the learner is not designed to perform well.

We next give a simple example illustrating when certification is possible and when it is not.
This example also serves as a building block in Theorem~\ref{thm:oig_erm_fail} to prove our negative results for ERM and OIG. Consider a hypothesis class on $[n]$ where every hypothesis has all but $\sqrt{n}$ points
labeled the same way; i.e., with at least $n-\sqrt{n}$ zeros or at least $n- \sqrt{n}$ ones. Under the uniform distribution on $[n]$, the \emph{majority learner} which always predicts the most frequently seen label in training has
expected error on the order of $1/\sqrt{n}$, even with a single labeled sample. This learner is, in a sense,
catered to the uniform distribution: it may have large error on marginals that put much larger
mass on points with the minority label.  
Certifying an error of $O(1/\sqrt{n})$ for the uniform distribution then hinges on whether such marginals can be distinguished from  uniform using unlabeled samples. This is possible using techniques from \emph{uniformity testing} precisely when the number of samples $m$ is at least on the order of $\sqrt{n}$. In other words, though the majority learner achieves error rate on the order of $1\sqrt{n}$ for all $m \geq 1$, this is only certifiable for $m =\Omega\left(\sqrt{n}\right)$. More generally, anytime a learner's error guarantee hinges on some property of the marginal distribution, the extent to which  certification is possible depends on whether the property can be detected from unlabeled data.

 We can now define relatively smart learning, a relaxation of smart learning which judges a learner relative to the best certifiable error rate for each distribution.  Informally speaking, \unsdreplace{relatively smart learning lets us}{this lets us} ``off the hook'' whenever small distribution-fixed errors cannot be certified from unlabeled data. 

\begin{definition}[Relatively Smart Learning]
\label{def:relatively_smart}
For a function $\sigma: \NN \times (0,1)\to \NN$ and constant $\alpha > 0$,  we call a learner $\A$ \emph{relatively $(\alpha,\sigma)$-smart} if \[\epsilon_\A(\D,\sigma(m,\eta)) \leq \alpha \epsilon(\D,m) + \eta\] for every certifiable distribution-dependent error rate  $\epsilon$, every distribution  $\D \in \DD$, every sample size $m \in \NN$, and every additive error parameter $\eta \in (0,1)$. We also say $\A$ is \emph{relatively smart} if there exist $\alpha$ and $\sigma$ such that it is relatively $(\alpha,\sigma)$-smart.
\end{definition}

Note that we allow a relatively smart learner's error to trail certifiable errors by a constant multiplicative term $\alpha$ and an additive term $\eta$, so long as $\eta$ can be made arbitrarily small.\footnote{Arbitrarily-small additive error $\eta$ features in our main positive result. It is an interesting and seemingly-challenging question whether fixing $\eta=0$ permits relatively smart learning.} We also allow the relatively smart learner to trail in the number of samples by an amount which can depend on $\eta$, as described by the \emph{sample blowup function} $\sigma(m,\eta)$. One could parameterize smart learning similarly by removing the terms ``certifiable'' and ``relatively'' from Definition~\ref{def:relatively_smart}, though the strong impossibility result of \cite[Theorem 2]{darnstadt_unlabeled_2013}---discussed in Section~\ref{sec:intro}---persists even with the allowances provided by $\alpha$, $\eta$, and~$\sigma$.

\unsddelete{
One might initially hope to construct relatively smart learners with  $\alpha = O(1)$ and $\sigma(m,\eta) = O(m)$ for each fixed $\eta>0$ in fairly general settings.
 This turns out to be asking too much. The interesting question is therefore which, if any, sample blowup functions $\sigma(m,\eta)$ permit relatively smart learning.}

\section{Distribution-Free Setting: OIG and ERM}
\label{sec:dfree_oig_erm}

We begin our exploration of relatively smart learning in the distribution-free setting with the familiar \emph{one-inclusion graph (OIG)} and \emph{empirical risk minimization (ERM)} learners. Our most notable result here is that the OIG learner is relatively smart with only a quadratic blowup in sample complexity. We show that this is essentially tight by way of a quadratic lowerbound which holds for both the OIG and ERM learners. We leave wide open whether ERM is relatively smart, and discuss associated challenges.

We use standard definitions for  ERM and OIG, which can be found in Appendix~\ref{app:oig_erm_def}. 
For purposes of arguments presented in this section, the reader need only keep in mind the following defining properties of the two learners:
\begin{itemize}
    \item ERM outputs a hypothesis consistent with the (realizable) training data, with ties among such hypotheses broken adversarially.
    \item For each unlabeled dataset $S=(x_1,\ldots,x_n) \in \X^*$, the OIG learner minimizes the worst-case \emph{transductive error}, also often known as \emph{leave-one-out error}, among all learners. The worst-case transductive error of a learner $\A$ on unlabeled dataset $S$ is defined as follows:
        \[
        \epsilon_{\A}^{\trans}(S) =\max_{h \in \H}  \frac{1}{n} \sum_{i=1}^n \id\left[\A\left(S_h^{(-i)}\right)(x_i) \neq h(x_i)\right],
    \]
    where $S_h = (x_1,h(x_1)), \ldots, (x_n,h(x_n))$ is the labeled dataset corresponding to $S$ and $h \in \H$, and the training data $S_h^{(-i)}$ consists of $S_h$ with its $i$th sample $(x_i,h(x_i))$ omitted.
\end{itemize}

We present our negative result for the OIG and ERM learners first. Though this  is subsumed by the more general impossibility result in Section~\ref{sec:dfree_lb}, and moreover technically modest, we find the associated arguments a suitable warmup for appreciating the nuances of relatively smart learning.
\begin{theorem}\label{thm:oig_erm_fail}
There exists a hypothesis class on a countable domain such that for every pair of constants $\alpha$ and $\beta$ and every function $\sigma(m,\eta) = O_\eta (m^{2 - \beta})$, the ERM and OIG learners fail to be relatively $(\alpha,\sigma)$-smart in the distribution-free setting.
\end{theorem} 
\noindent Notably, the certifiable error rates with which ERM and OIG cannot compete are achieved by the simple \emph{majority learner}, which just outputs the most frequent label in the training data. 

The formal proof of Theorem~\ref{thm:oig_erm_fail} can be found in \ifthenelse{\boolean{coltsubmission}}{Appendix}{Section}~\ref{pf:oig_erm_fail}, but the high-level idea is as follows. 
Consider the following hypothesis class on $n$ datapoints: At least $n - o(n)$ datapoints have the same \emph{majority label} (whether $0$ or $1$), with the remaining $o(n)$ points having the other \emph{minority label}. Think of the number of minority labels as being only slightly sublinear, e.g. $n^{0.99}$. When faced with a uniform (or nearly uniform) distribution over the $n$ points, the majority learner quickly approaches a vanishing error rate $o(n)/n = o(1)$, even with very few samples $m$. A certifier can identify such a nearly uniform distribution on $n$ points when $m$ modestly exceeds $\sqrt{n}$---intuitively, this follows from the Birthday paradox, though making it formal requires appealing to uniformity testing. Therefore, the majority learner has a vanishing certifiable error rate starting around $m = \sqrt{n}$ samples. The OIG and ERM learners, in contrast, suffer constant error until their number of samples exceeds the number of minority labels, which is only slightly sublinear in $n$. To see why this is the case, note that any set of points smaller than the number of minority labels is shattered by the hypothesis class. Taking the disjoint union of this construction over all integers~$n$, this essentially rules out any subquadratic bound on the blowup in sample complexity needed for ERM/OIG to catch up to the certifiable error rate of the majority learner.

We now present our positive result, which shows that a quadratic blowup in sample complexity---independent of the marginal---suffices to compete with certifiable semi-supervised guarantees. 
\begin{theorem}\label{thm:oig_square_smart}
For every domain and hypothesis class, the OIG  learner is relatively $\left(e,e\frac{m^2}{\eta}\right)$-smart in the distribution-free setting.
\end{theorem}
Theorem \ref{thm:oig_square_smart} stands in contrast to the impossibility result of \cite[Theorem 2]{darnstadt_unlabeled_2013}, as described in Section~\ref{sec:intro}.\footnote{While phrased for a family of marginals, their result implies the same impossibility in the distribution-free setting. This is because the benchmark in smart learning (unlike ours) does not depend on the distribution class as a whole.}   
This impossibility of smart learning can therefore be attributed to the challenge of certifying distribution-dependent error rates from unlabeled data. 
Taking the contrapositive, a simple corollary of Theorem~\ref{thm:oig_square_smart} is that sound certification---by way of a learner/certifier pair for each distribution---of near-optimal distribution-fixed error rates suffices for smart learning. 

We formally prove Theorem~\ref{thm:oig_square_smart} in Section~\ref{pf:oig_square_smart}, but the high-level idea is as follows. Consider a learner $\A$ specialized to a marginal $\D$, and let $\C$ be its certifier. When allowed $m$~samples, the certifier cannot distinguish between $\D$ and the uniform distribution $\D'$ on $S$, where $S$ consists of $M=c m^2$ i.i.d.~samples from $\D$ for a sufficiently large constant $c$. Intuitively, this follows from the Birthday paradox, though it takes some technical work to make it precise.\footnote{We note that a similar argument is employed in \cite[Theorem 6.2]{gollakota2023moment}.} Since we require soundness of our certifiers, any $m$-sample certifiable error for $\D$ can be no better than the best error attainable on $\D'$ with $m$ samples. The OIG learner---being dataset-by-dataset optimal in the leave-one-out sense---can be shown competitive with this when given roughly $M$ i.i.d.~samples from $\D$, as those correspond to what is effectively a constant fraction of the support of $\D'$. This yields the result. \looseness=-1

It is natural to wonder whether ERM---or, for that matter, any learner that is simpler and typically more tractable than OIG---is also relatively smart with some finite (perhaps even quadratic) blowup in sample complexity. We however leave this question wide open.
\begin{open}\label{open:erm}
Is ERM relatively smart in the distribution-free setting? Failing that, what about other natural and tractable learners?
\end{open}
 The sample complexities of ERM and OIG are closely related in traditional PAC learning, so it is tempting to suspect something similar here. The challenge in proving such a statement, however, comes from the fact that fine-grained dataset-by-dataset comparisons between OIG and ERM, or quantities like the VC dimension, appear bound to falter. It is known that there are hypothesis classes where ERM drastically trails the OIG learner in leave-one-out error on some datasets---consider for example the behaviors of hamming weight at most one on a large dataset, and an ERM learner which breaks ties against the all-zero behavior. Similarly, there are hypothesis classes and arbitrarily large unlabeled datasets $S$ where the VC dimension of induced behaviors is $|S|-1$, yet the OIG learner achieves zero leave-one-out error. As an example of this, consider the \emph{parity} class on $S$, which ensures that the labels on $S$ sum to $0 \pmod{2}$. It therefore appears that any comparison between the error rates of the two learners cannot be argued dataset-by-dataset using leave-one-out arguments  or the VC dimension. 
 This suggests that either new proof approaches are needed, or ERM may not be relatively smart after all.

\ifthenelse{\boolean{coltsubmission}}{}{
\ifthenelse{\boolean{coltsubmission}}{\section{Formal Proof of Theorem~\ref{thm:oig_erm_fail}}
\label{pf:oig_erm_fail}}{\subsection{Formal Proof of Theorem~\ref{thm:oig_erm_fail}}
\label{pf:oig_erm_fail}}
Fix an arbitrary constant $\beta \in \left(0, \frac{1}{8}\right)$. We will construct a hypothesis class $\HH$ for which OIG and ERM are not relatively smart for any  $\sigma(m,\eta) = O_{\eta}(m^{2-14 \beta})$. Consider the domain $\X \subseteq \NN \times \NN$ with $\X := \bigcup_{n \in \NN} \X_n$ and $\X_n: = \{n\} \times [n]$ for all $n\in\NN$. We will sometimes refer to $\X_n$ as the \emph{$n$th row} of the domain. Denote by $\D(n)$ the uniform distribution on the $n$th row $\X_n$. We will define a hypothesis class $\H(n)$ such that the probability under $\D(n)$ of minority label of any $h\in \H(n)$ is about $n^{-\beta}$.
Define the functions $M,m:\NN \rightarrow \NN$ and $\xi: \rightarrow (0,1)$ as 
\begin{equation}\label{eq:m(n)xi(n)}
M(n) = \lceil n^{1-\beta}\rceil,
m(n) = \lfloor n^{\frac{1}{2}+3\beta} \rfloor, \,\text{and}\, \xi(n) = M(n)/n \approx n^{-\beta}.
\end{equation}
For $b\in\{0,1\}$ and $n \in \NN$, let 
\[
\begin{aligned}
&\H^{(b)}(n) = \{h \in \{0,1\}^{\X}: \forall x\in \X\setminus \X_{n}, h(x)=1 \,\text{and}\,|\{x\in\X_{n}: h(x)=b\}|= M(n)\},
\end{aligned}
\]
and define $\H(n) = \H^{(0)}(n) \cup \H^{(1)}(n)$.  
In other words, for each $h\in \H(n)$ there are exactly $M(n)$ points in row $n$ with the minority label, whereas  $h$ is one everywhere outside that row. Note that $\xi(n)$ is the probability under $\D(n)$ of a minority label for any $h\in\H(n)$, which approaches $0$ as $n$ grows. We let $\HH:=\bigcup_{n \in \NN} \H(n)$. 

Let $\A_{\maj}$ be the learner which ignores unlabeled data, and always predicts the most frequent label seen in training; formally, $\A_{\maj}(S)(x') = \id\left[|\{(x,y) \in S : y =1\}| \geq |S|/2 \right]$. We show that $\A_{\maj}$ has error $O(\xi(n))$ on $\HH$ with respect to $\D(n)$ and sample sizes $m\geq m(n)$ and, more importantly, this error rate is certifiable. We prove the following.
\begin{lemma}\label{lemma:certifiable_majority_error}
There exists an absolute constant $C=C(\beta)\in \NN$ such that for all $n\geq C$, the rate 
\[\epsilon_{n}(\D,m) = \begin{cases}
    4\xi(n) & \text{if $\D = \D(n)$ and $m \geq m(n)$}\\
    1 & \text{otherwise.}
\end{cases}\]
is certifiable for $\A_{\maj}$ in the distribution-free setting.
\end{lemma}
We prove Lemma~\ref{lemma:certifiable_majority_error} by certifying the error rate of $\A_{\maj}$. Specifically,  for each row $n \in \NN$ we exhibit a certifier $\C_{\maj,n}$ (Algorithm~\ref{alg:certifier_majority}) satisfying \(\epsilon_{\A_{\maj}}(\D,m)\leq \expects{S \sim \D^m}{\C_{\maj,n}(S)}\leq \epsilon_{n}(\D,m)\) for all $m \in \NN$ and all distributions $\D$ on $\X$. Note that for distributions on the row of interest $n$, $\A_{\maj}$ will only do well for distributions close to uniform, so our certifier must output small values exclusively for those. We therefore exploit uniformity testing---recapped in Appendix~\ref{app:uniformity_testing}---to distinguish the uniform distribution on row $n$ from distributions a substantial total variation distance away. Our uniformity tester adds a simple wrapper around  the standard uniformity tester $\tunifstandard$ (Algorithm~\ref{alg:uniformity_test_standard}), which allows detecting distributions with support outside the desired row. The resulting modified uniformity tester $\tunif$ is shown in Algorithm~\ref{alg:uniformity test} with guarantees summarized in Lemma~\ref{lemma:uniformity_tester}. The remaining technical details for proof of Lemma~\ref{lemma:certifiable_majority_error} are fairly standard, and therefore relegated to Appendix~\ref{pf:certifiable_majority_error}.

\begin{algobox}{Certifier $\C_{\maj,n}(S)$ run on input sample $S$:}\label{alg:certifier_majority}
    \begin{enumerate}[itemsep=0pt]
    \item Let $O=\tunif_{\X_n}(\xi(n),\xi(n),S)$.\Comment{Algorithm~\ref{alg:uniformity test} for testing against $\D(n)$}
    \item If $|S| \geq m(n)$ and $O=1$ then return $3\xi(n)$, else return 1.
\end{enumerate}
\end{algobox}

Having upper-bounded the certifiable error rate of the majority learner for $\D(n)$ with $m \geq m(n)$ samples (Lemma~\ref{lemma:certifiable_majority_error}), it remains to lower bound the error of ERM and OIG learners in the same regime. The proof of the following Lemma is fairly elementary and  relegated to Appendix~\ref{pf:lower_bound_erm_oig}.
\begin{lemma}\label{lemma:lower_bound_erm_oig}
     For the hypothesis class $\HH$ we have $\epsilon_{\A_{\erm}}(\D(n), m) \geq 1-2\xi(n)$ and $\epsilon_{\A_{\oig}}(\D(n),m) \geq 1/2 - 2n^{-\beta/2} = 1/2-o(1)$ for all $m < M(n)$.
\end{lemma}

For the distribution $\D(n)$, Lemmas~\ref{lemma:certifiable_majority_error} and~\ref{lemma:lower_bound_erm_oig} imply that it would take at least $M(n)\approx n^{1-\beta}$ samples for the error of OIG or ERM to approximate---to within any multiplicative constant $\alpha$  and any additive constant $\eta<\frac{1}{2}$, both independent of $n$---the certifiable error rate of the majority learner with $m(n) \approx n^{\frac{1}{2} + 3 \beta}$ samples. Since $M(n) = \omega\left((m(n))^{2-14\beta}\right)$ 
this rules out relative $(\alpha,\sigma)$-smartness of OIG and ERM for any constant $\alpha$ and any $\sigma(m,\eta) = O_\eta(m^{2-14\beta})$.\hfill \qed

\begin{remark}\label{remark:variable_beta}
    The preceding proof of Theorem~\ref{thm:oig_erm_fail} may appear to require a different hypothesis class $\HH_\beta$ for every choice of the parameter $\beta$. This can be avoided by taking the disjoint union of countably many $\HH_\beta$---over disjoint countable domains $\X_\beta$---for a sequence of $\beta$ values tending to zero. In more detail, for each $\beta \in B=\{1/i: i \in \NN,i>8\}$ we define $\HH_\beta$ over its own distinct domain $\X_\beta$ as in the preceding proof, let $\X= \biguplus_{\beta \in B} \X_\beta$, and extend each $h \in \HH_\beta$ to $\X$ canonically by setting $h(x) = 1$ for $x \not\in \X_\beta$. The resulting single hypothesis class $\HH = \biguplus_{\beta \in B} \HH_\beta$ over the countable domain $\X$ serves to witness the impossibility result of Theorem~\ref{thm:oig_erm_fail}, as needed. 
\end{remark}

}

\subsection{Formal Proof of Theorem~\ref{thm:oig_square_smart}}
\label{pf:oig_square_smart}

Let $\epsilon(.,.)$ be any certifiable error rate. Fix a distribution $\D$ and a sample size $m \geq 3$. We prove that \( \epsilon_{\A_{\oig}}(\D,\frac{e}{\eta} m^2) \leq e^{1-\eta/3e}\epsilon(\D,m) + \eta
    \) for every $\eta \in \left(0,1\right)$.
    Since $\epsilon(.,.)$ is certifiable, by definition there exists a learner $\A = \A_\D$ catered to $\D$ and a corresponding sound certifier $\C= \C_\A$ such that (i) $\epsilon_{\A}(\D,m) \leq \expects{S \sim \D^m}{\C(S)} \leq \epsilon(\D,m)$ and (ii) for any distribution $\D'$, $\epsilon_{\A}(\D',m) \leq \expects{S \sim {\D'}^m}{\C(S)}$, where the second fact crucially relies on soundness of the certifier. 

     It is easy to observe that the process of sampling $m$ i.i.d.~samples from $\D$ is equivalent to sampling a larger multi-set $S\sim {\D}^M$ of size $M$ and then drawing a multi-set $T\simnr \D_{S}^{m}$ of $m$ samples from $S$ uniformly at random \emph{without replacement}. Denote by $\azreplace{\Gamma(m,M)}{\Gamma} :=  \frac{M(M-1)\ldots (M-m+1)}{M^{m}}$ the probability of observing no duplicates when making $m$ independent draws (with replacement) from a uniform distribution on $M$ points. Let $M = c m^2$ for a constant $c=c(\eta)$ to be chosen later. Then,
     \ifthenelse{\boolean{coltsubmission}}{\begin{equation}\label{eq:epsilon_to_eps_A}
     \begin{aligned}
    \epsilon(\D,m)\geq  \expects{S \sim \D^{m}}{\C(S)} = \expected{S \sim \D^M}\expects{T \simnr \D_{S}^{m}}{\C(T)}
        &\geq  \frac{1}{\Gamma}\expected{S \sim \D^M}{\expects{T \sim \D_{S}^{m}}{\C(T)}} + 1- \frac{1}{\Gamma}\\
        & \geq  \frac{1}{\Gamma}\expects{S \sim \D^M}{\epsilon_{\A}(\D_S,m)}+1- \frac{1}{\Gamma},
     \end{aligned}
          \end{equation}}
          {\begin{equation}\label{eq:epsilon_to_eps_A}
     \begin{aligned}
    \epsilon(\D,m)\geq  \expects{S \sim \D^{m}}{\C(S)} &= \expected{S \sim \D^M}\expects{T \simnr \D_{S}^{m}}{\C(T)}
        \\&\geq  \frac{1}{\Gamma}\expected{S \sim \D^M}{\expects{T \sim \D_{S}^{m}}{\C(T)}} + 1- \frac{1}{\Gamma}\\
        & \geq  \frac{1}{\Gamma}\expects{S \sim \D^M}{\epsilon_{\A}(\D_S,m)}+1- \frac{1}{\Gamma},
     \end{aligned}
          \end{equation}}
    where the first inequality is due to  property (i), the second inequality is by simple algebraic manipulation and the fact that certifier outputs are at most 1, and  the last inequality is due to property (ii). \looseness=-1 

    \unsdreplace{For any distribution $\D'$ and sample size $m'$
    denote by $\epsilon^*(\D',m')=\inf_{\A} \eps_\A(\D',m')$ the minimum error of any learner on $\D'$ when the input is $m'$ i.i.d. draws from $\D'$ (a.k.a.~the optimal distribution-fixed error for $\D'$ with $m'$ samples).}{Let $\eps^*(.,.)$ denote optimal distribution-fixed error rates.} Building on Equation~\eqref{eq:epsilon_to_eps_A} we get
    \ifthenelse{\boolean{coltsubmission}}{\begin{equation}\label{eq:eps_to_epsstar}
     \begin{aligned}
         \eps(\D,m)  \geq \frac{1}{\azreplace{\Gamma(m,M)}{\Gamma}}\expects{S \sim \D^M}{\epsilon_{\A}(\D_S,m)} + 1- \frac{1}{\azreplace{\Gamma(m,M)}{\Gamma}} & \geq  \frac{1}{\azreplace{\Gamma(m,M)}{\Gamma}}\expects{S \sim \D^M}{\epsilon^*(\D_S,m)} + 1- \frac{1}{\Gamma}\\
        & \geq  \frac{1}{\azreplace{\Gamma(m,M)}{\Gamma}}\expects{S \sim \D^M}{\epsilon^*(\D_S,M-1)} +  1- \frac{1}{\Gamma},
     \end{aligned}
    \end{equation}}
    {\begin{equation}\label{eq:eps_to_epsstar}
     \begin{aligned}
         \eps(\D,m)  & \geq \frac{1}{\azreplace{\Gamma(m,M)}{\Gamma}}\expects{S \sim \D^M}{\epsilon_{\A}(\D_S,m)} + 1- \frac{1}{\azreplace{\Gamma(m,M)}{\Gamma}}\\
         & \geq  \frac{1}{\azreplace{\Gamma(m,M)}{\Gamma}}\expects{S \sim \D^M}{\epsilon^*(\D_S,m)} + 1- \frac{1}{\Gamma}\\
        & \geq  \frac{1}{\azreplace{\Gamma(m,M)}{\Gamma}}\expects{S \sim \D^M}{\epsilon^*(\D_S,M-1)} +  1- \frac{1}{\Gamma},
     \end{aligned}
    \end{equation}}
    where the last inequality follows from the monotonicity of optimal distribution-fixed rates. 

Next we upperbound the error of the OIG learner on $\D$ in terms of the expected optimal distribution-fixed error on an empirical distribution $\D_S$, for a sample $S$ drawn i.i.d.~from $\D$. This is articulated in the following Lemma. 
\begin{lemma}\label{lemma:expected_to_transductive}
       For any distribution $\D$ and sample size $m$, denote by $\epsilon^*(\D,m)$ the minimum error of any learner on $\D$ when given $m$ i.i.d. samples. Then
        \(
        \epsilon_{\A_{\oig}}(\D,M-1)\leq e\expects{S \sim \D^M}{\epsilon^*(\D_S,M-1)}.
        \)
    \end{lemma}

To prove Lemma~\ref{lemma:expected_to_transductive} we first use the reduction from transductive to PAC learning employed in the proof of \cite[Lemma~A.1]{asilis2025proper} (similar also to \cite[Lemma~34]{asilis_regularization_2024}) to conclude that for each unlabeled dataset $S$ of size $M$ there is a learner with worst-case transductive error on $S$ bounded by $e\cdot \epsilon^*(\D_S,M-1)$. We then invoke the fact that OIG minimizes worst-case transductive error (recall Definition~\ref{def:oig}) to conclude that the transductive error of OIG on $S$ is also at most $e\cdot \epsilon^*(\D_S,M-1)$. The lemma then follows from the standard leave-one-argument which upperbounds any learner's expected error in the PAC setting  with $M-1$ training samples by its expected transductive error on datasets of size $M$ drawn from the same distribution.
Remaining technical details for proof of Lemma~\ref{lemma:expected_to_transductive} appear in Appendix~\ref{app:OIG_smart}.

    Combining Lemma~\ref{lemma:expected_to_transductive} with Equation~\eqref{eq:eps_to_epsstar} we get that
     \[
     \begin{aligned}
         \epsilon(\D,m)
        & \geq \frac{1}{e \cdot \Gamma}\epsilon_{\A_{\oig}}(\D,M-1) +  1- \frac{1}{\Gamma}.
     \end{aligned}
          \]
    Finally, for $c>2$ and $m\geq 3$ we use the fact that $1 - 1/c \leq \Gamma \leq e^{-1/3c}$ to conclude that 
    \[
    \epsilon_{\A_{\oig}}(\D,M-1) \leq  e\cdot \Gamma \cdot \epsilon(\D,m) + e - e\cdot \Gamma \leq e^{1-1/3c}\epsilon(\D,m) + e/c.
    \]
    Setting $c = \frac{e}{\eta}$ then yields
    \(
    \epsilon_{\A_{\oig}}(\D,M-1) \leq e^{1-\eta/3e} \epsilon(\D,m) + \eta
    \), as desired.\hfill\qed

\section{Distribution-Free Setting: An Impossibility for all Learners}
\label{sec:dfree_lb}
Section~\ref{sec:dfree_oig_erm} established that in order to compete with every certifiable semi-supervised guarantee, a quadratic blowup in sample complexity is necessary for the OIG and ERM learners, and sufficient for OIG.
Can a different learner do better? Theorem~\ref{thm:oig_erm_fail} leaves open this possibility, as the ``hard marginals'' from that construction are all amenable to a single learner, namely majority.  Nonetheless, we show by way of a more intricate hypothesis class that the answer in general is no, and a quadratic blowup in sample complexity is in fact the best possible. We prove the following~theorem. 
\begin{theorem}\label{thm:no_learner_m2_smart}
 There exists a hypothesis class on a countable domain such that for every pair of constants $\alpha$ and $\beta$ and every function $\sigma(m,\eta) = O_\eta (m^{2 - \beta})$,  no learner is relatively $(\alpha,\sigma)$-smart.
\end{theorem}

The formal proof of Theorem~\ref{thm:no_learner_m2_smart} can be found in Section~\ref{pf:no_learner_m2_smart}, but the high-level idea is as follows. For each integer $n$ and parameter $\beta$, we build a set system $\S= \S(n,\beta)$ on a large domain with $|S| = n$ for each $S \in \S$ and $|S \cap S'| \leq n^{1-O(\beta)} = o(n)$ for distinct $S,S' \in \S$.  A hypothesis class $\H = \H(n,\beta) = \{h_S : S \in \S\}$ is then defined so that $h_S$ is effectively a random behavior on~$S$, but constant elsewhere (and therefore constant on most of any other $S' \in \S$). 

First, we show that a learner catered to the uniform distribution on some known $S \in \S$ achieves vanishing error with very few samples $m$, and can certify this guarantee using uniformity testing on $S$ when $m$ modestly exceeds $\sqrt{|S|} = \sqrt{n}$. (The certifier outputs the trivial bound, effectively eschewing any guarantee, if the uniformity tester fails). This is possible because only $h_S$ is  ``interesting'' on $S$, whereas all other $h_{S'}$ with $S' \neq S$ are constant on all but $|S \cap S'| = o(n)$ points in $S$, and can therefore be ``spotted'' with few samples (much akin to our analyses of the majority learner from Theorem~\ref{thm:oig_erm_fail}). Second, we show that $\H$ can be made sufficiently rich to shatter every set of size roughly $n^{1-O(\beta)}$, which prohibits any meaningful learning with less than that many samples when $S$ is unknown.
Taken together, these two properties imply the theorem for a fixed parameter~$\beta$ and number of samples $m\approx \sqrt{n}$. The theorem then follows by taking the  disjoint union of the hypothesis classes $\H(n,\beta)$ for all integers $n$ and a countable sequence  of $\beta$s tending to~$0$.  \looseness=-1  

\subsection{Formal Proof of Theorem~\ref{thm:no_learner_m2_smart}}
\label{pf:no_learner_m2_smart}

Fix an arbitrary constant $\beta\in(0,\frac{1}{8})$. We will construct a set system $\S$ and a hypothesis class $\HH$ that witness the impossibility of relative smartness of any learner for $\sigma(m,\eta) = O_{\eta}(m^{2-14\beta})$. 
Using the probabilistic method, we  construct set systems $\S(n)= \S(n,\beta)$ and hypothesis classes $\H(n)=\H(n,\beta)$ satisfying certain properties for  sufficiently large $n\in\NN$. We then take the disjoint union over them to create $\S$ and $\HH$. 

The following lemma, proved in Appendix~\ref{pf:set_system},  summarizes the properties of $\S(n)$ we exploit.
\begin{lemma}\label{lemma:set_system}
   There exists an absolute constant $C = C(\beta)\in\NN$ such that for all $n\geq C$,
     there exists a set system $\S(n)$ on a universe $\U(n)$ with the following properties. 
     \begin{enumerate}[label=(\roman*)]
        \item\label{property:set_system_universe_size} $|\U(n)| = \lceil n^{1+\beta}\rceil$  and $|\S(n)| = \lceil\exp\left(\frac{1}{4}n^{1-\beta/2}\right)\rceil$ 
         \item \label{property:set_size} $|S| = n$ for all $S \in \S(n)$
         \item \label{property:intersection_size}$|S \cap S'| \leq n^{1-\beta/2}$ for any distinct $S, S' \in \S(n)$
         \item\label{property:container_size} Every subset $T \subset \U(n)$ with size $|T| \leq 2n^{1-\beta}$ is contained in at least $\frac{1}{2}\exp\left(\frac{1}{8}n^{1-\beta/2}\right)$ many sets in $\S(n)$, i.e., $| \{S \in \S(n) : T \subseteq S\}| \geq \frac{1}{2}\exp\left(\frac{1}{8}n^{1-\beta/2}\right)$.
     \end{enumerate} 
\end{lemma}
For each $T \subseteq \U(n)$ we use $\S_T := \{S \in \S(n): T\subseteq S\}$ to denote the collection of sets in $\S(n)$ containing~$T$. Our hypothesis class $\H(n)$ satisfies properties outlined in the following lemma, whose proof can be found in Appendix~\ref{pf:labeling_set_system}.
\begin{lemma}\label{lemma:labeling_set_system}
    There exists an absolute constant $C' = C'(\beta)\in \NN, C'\geq C$ such that for all $n\geq C'$ there exists a hypothesis class
     $\H(n):= \{h_S: S \in \S(n)\} \subseteq \{0,1\}^{\U(n)}$ with the following properties.
     \begin{enumerate}[label=(\roman*)]
         \item \label{property:constant_outside}For all $S \in \S(n)$ and $x \notin S$, $h_S(x) = 1$.
          \item \label{property:balanced_container} For every subset $T \subset \U(n)$ with $|T|\leq 2n^{1-\beta}$ and every binary labeling $b \in \{0,1\}^{|T|}$,  
         \[(1-n^{-\beta}) \frac{|\S_T|}{2^{|T|}}  \leq |\{h_S \in \H(n): S \in \S_T, h_S|T = b\}| \leq (1+n^{-\beta}) \frac{|\S_T|}{2^{|T|}}.\] 
     \end{enumerate} 
\end{lemma}
The hypothesis class $\HH=\HH_\beta$ is now defined by taking the disjoint union of $\H(n)$ over all $n\geq C'$ as follows. Define the universe $\U:= \biguplus_{n\geq C'}\U(n)$ and let $\S:= \biguplus_{n\geq C'}\S(n)$. Let $\overline{\H}(n)$ be the extension of $\H(n)$ where each $h\in\overline{\H}(n)$ is constant $1$ on $\U \setminus \U(n)$. Finally, let $\HH:= \biguplus_{n\geq C'} \overline{\H}(n)$.

Define $m(n) = \lfloor n^{1/2+3\beta} \rfloor$ and $\xi(n) = n^{-\beta/2}$.
Observe that for each $S\in \S(n)$,  each $h_{S'}$ with $S'\neq S$ assigns the label $1$ to every point in $S$ except possibly to $S\cap S'$, which under $\D_S$ has probability mass at most $|S \cap S'|/n \leq n^{-\beta/2}= \xi(n)$. Therefore, one of $h_S$ or the constant $1$ predictor incurs error at most $\xi(n)$ under $\D_S$, and validating between them using $m \geq \Omega\left(\frac{\ln(1/\xi(n))}{\xi(n)}\right)$ samples suffices for guaranteeing error $O(\xi(n))$ when $\D_S$ is fixed and known. 

Sound certification of such a bound faces two interrelated obstacles, however: (a) The certifier must recognize when the true distribution $\D'$ is far from $\D_S$ and output a valid upperbound on the learner's error, and (b) the learner's error on each $\D'$ must be sufficiently bounded away from~$1$ to accommodate the difficulty inherent to (a). 
We therefore construct, for each $S \in \S(n)$, a learner $\A_S$ (Algorithm~\ref{alg:s-fixed-learner}) and certifier $\C_S$ (Algorithm~\ref{alg:s-fixed-certifier}) that are jointly designed to satisfy both requirements. For (b), we ensure errors are bounded away from 1 for \emph{every} distribution by adding the constant~0 and constant~1 predictors into consideration by the learner. In other words, $\A_S$ validates between $h_S$ and the majority learner. For (a), much like in our proof of Theorem~\ref{thm:oig_erm_fail}, $\C_S$ employs uniformity testing to recognize distributions close to $\D_S$, and outputs small values only in those cases. Since uniformity testers are only accurate for sufficiently large $m$, our soundness guarantee is restricted to $m \geq m(n)$. In summary, $\C_S$ soundly certifies error $O(\xi(n))$ for $\A_S$ when given $m \geq m(n)$ samples from $\D_S$. This is articulated in the following lemma, formal proof of which can be found in Appendix \ref{pf:cert_error_rate_no_learner_smart}.

\begin{lemma}\label{lemma:cert_error_rate_no_learner_smart}
   There exists an absolute constant $C''=C''(\beta)\geq C'$ such that for all $n\geq C''$ the rate
   \[
   \epsilon_{n}(\D,m) = 
   \begin{cases}
        7n^{-\beta/2} & \D \in \{\D_S: S \in \S(n)\}, m\geq m(n) \\
        1 & \text{otherwise.}
   \end{cases}
   \]
    is certifiable in the distribution-free setting.
\end{lemma}

\begin{algobox}{Learner $\A_{S}(T)(x)$ trained on sample $T$ given test point $x$:}\label{alg:s-fixed-learner}
    \begin{enumerate}[itemsep=0pt]
    \item Randomly split $T$ into two  equal parts $T_1$ and $T_2$, i.e., $T= T_1 \uplus T_2$ with $|T_1|=\left\lceil \frac{|T|}{2} \right\rceil$. 
      \item If $L(h_S,T_2)\leq L(\A_{\maj}(T_1),T_2)$ then return $h_S(x)$;
    \item Else return $\A_{\maj}(T_1)(x)$.
\end{enumerate}
\end{algobox}

\begin{algobox}{Certifier $\C_{S}(T)$ run on (unlabeled) input sample $T$:}\label{alg:s-fixed-certifier}
    \begin{enumerate}[itemsep=0pt]
    \item Let $n=|S|$ and $O=\tunif_{S}(\xi(n),\xi(n),T)$. \Comment{Algorithm~\ref{alg:uniformity test} for testing against $\D_S$}
    \item If $|T|\geq m(n)$ and $O=1$ then return $6\xi(n)$;
    \item Else return $1$.
\end{enumerate}
\end{algobox}

It remains to show that no learner can be smart relative to the above certifiable error rate. We prove that any learner $\A$, on average over all distributions in $\{\D_S:S\in\S(n)\}$, has error at least $1/2-2\xi(n)^2$ when given at most $n^{1-\beta}$ samples. This implies that for each $m\leq n^{1-\beta}$, there exists a set $S^*=S^*_{n,m}\in\S(n)$ such that $\epsilon_{\A}(\D_{S^*},m)\geq 1/2 - 2\xi(n)^2$.
  The key idea is that for any sample of size at most $n^{1-\beta}$ and any fixed test point, roughly half of the distributions $\D_S$ consistent with this sample–test pair---namely, those for which $S$ contains both the sample and the test point, and $h_S$ is consistent with the sample---label the test point differently from the learner’s prediction. Unlike the lower bound of PAC learning where the distributions are defined on shattered sets and the desired conclusion follows directly from the definition, the sets in our system can be much larger than the size of shattered sets. Property~\ref{property:balanced_container} in Lemma~\ref{lemma:labeling_set_system} is introduced precisely to ensure that such a guarantee continues to hold in this setting.
  This is formally captured in the following lemma, the proof of which is relegated to Appendix~\ref{pf:any-learner-average-error}.
\begin{lemma}\label{lemma:any-learner-average-error}
      Let $\A: (\U\times\{0,1\})^*\times \U \rightarrow \{0,1\}$ be any learner. For any $n\geq C''$ and  $m\leq n^{1-\beta}$ there exists a set $S^*=S^*_{n,m} \in \S(n)$ such that  $\epsilon_{\A}(\D_{S^*},m) \geq 1/2 -2n^{-\beta}$. 
\end{lemma}

Lemmas~\ref{lemma:cert_error_rate_no_learner_smart} and \ref{lemma:any-learner-average-error} conclude that any fully-supervised learner requires at least $n^{1-\beta}$ many samples for its error to approximate---within any multiplicative constant $\alpha$  and \unsddelete{any} additive constant $\eta<\frac{1}{2}$, both independent of $n$---the (certifiable) error rates of $\D_S$-fixed learners for $S \in \S(n)$.
It follows that no learner is relatively $(\alpha, \sigma(m,\eta))$-smart for any constant $\alpha$ and $\sigma(m)= O_{\eta}(m^{2-14\beta})$. We conclude by taking a disjoint union of countably many $\HH_{\beta}$---with disjoint domains $\X_\beta$---for a sequence of $\beta$ tending to zero (e.g. $\beta \in \{1/i: i \in \NN, i>8\}$), where each $h \in \HH_\beta$ is canonically extended to $\X$ by setting $h(x)=1$ for $x \not\in\X_\beta$. See  Remark~\ref{remark:variable_beta} for more detail.\hfill \qed 

\section{Distribution Families}
\label{sec:dfamily}
In this Section we examine relatively smart learning in distribution-family settings. We begin with the simple observation that our proof of Theorem~\ref{thm:oig_square_smart} only required a simple closure property of the distribution family---namely, closure under taking empirical distributions. This captures, for example, any distribution family defined by a family of allowable manifolds, where the only restriction is that each distribution's support must lie inside one of the manifolds.

\begin{corollary}\label{cor:oig_smart_family}
    Fix an arbitrary hypothesis class on some domain. Let $\DD$ be a distribution family with the following closure property: For every $\D \in \DD$ and every finite set $S$ contained in the support of $\D$, we also have $\D_S \in \DD$. The OIG learner is relatively $\left(e,e\frac{m^2}{\eta}\right)$-smart over $\DD$.
\end{corollary}

Beyond such ``simple'' distribution families, relatively smart learning can start to behave quite differently from the distribution-free setting. There are distribution families where smart (and therefore relatively smart) learning is possible, but neither OIG nor ERM is relatively smart.   Furthermore, there are families where relatively smart learning is not possible at all. The following pair of theorems summarize these findings. 

\begin{theorem} \label{thm:dfamily_oig_erm_notsmart}
There is a hypothesis class on a countable domain, and a countable distribution family on that domain, such that the following hold with respect to the family: There is a $(1,m)$-smart learner, yet neither OIG nor ERM is relatively smart.\footnote{{We mean that there exist no  $\alpha$ and $\sigma$ such that the learner is relatively $(\alpha,\sigma)$-smart. (Recall Definition~\ref{def:relatively_smart}).}}
\end{theorem}
\begin{proof}
We show that any distribution family which is not smartly learnable (in the non-relative sense) can be modified to satisfy the theorem. Such a countable distribution family $\DD$ exists for a hypothesis class $\H$ on a countable domain $\X$, as shown in \cite[Theorem 2]{darnstadt_unlabeled_2013}.

Define a new domain $\X'$ which ``tags'' each $x \in \X$ with the name of a distribution $\D \in \DD$; formally, $\X' = \X \times \DD$. Now extend $\H$ to $\X'$ by simply ignoring the tags; i.e., for each $h \in \H$ define $h'(x,\D) = h(x)$ and $\H'= \{h' : h \in \H\}$. Finally, port each distribution $\D \in \DD$ to the subset of the domain tagged with $\D$, preserving probabilities; i.e., let $\D'$ be such that $\probs{\D'}{(x,\D)} = \probs{\D}{x}$, and let $\DD' = \{\D' : \D \in \DD\}$. Clearly $\X'$ and $\DD'$ are countable.

Any sample from a distribution in $\D' \in \DD'$ uniquely identifies it, so the learner which identifies $\D'$ then implements the optimal distribution-fixed learner for it is $(1,m)$-smart. For the same reason, the optimal distribution-fixed error for each $\D' \in \DD'$, equal to the optimal distribution-fixed error for its precursor $\D \in \DD$, is soundly certifiable. Any learner which ignores the tags, such as ERM or OIG, does no better on $\D'$ than on its precursor $\D$, for any sample size. Recalling that $\DD$ is not smartly learnable, we conclude that no such learner is relatively smart over $\DD'$.
\end{proof}
The tag construction should be viewed as a deliberately simple example of a broader
phenomenon. Whenever marginal distributions in the family can be learned well from unlabeled data---of which recognizing a unique tag is the most trivial example---and moreover this aids in subsequent supervised learning, algorithms such as ERM or OIG which do
not perform such unsupervised learning of the marginal need not be relatively smart.

\begin{theorem} \label{thm:dfamily_nosmart}
There is hypothesis class on a countable domain, and a countable distribution family~$\DD$ on that domain, such that no learner is relatively smart  over~$\DD$.
\end{theorem}
\begin{proof}    
To rule out relatively smart learning, we show that it suffices for $\DD$ to satisfy two properties: (i) $\DD$ does not admit a smart learner, and (ii) $\DD$ is \emph{well separated} in that there exists an absolute constant $c\in [0,1)$ so that  $\probs{\D'}{\supp(\D)} < c$ for any distinct\unsddelete{ pair of distributions} $\D,\D' \in \DD$.

Consider $\DD$ satisfying (i) and (ii). For $\D \in \DD$ let $\A=\A_\D$ be its  optimal distribution-fixed learner, and let $\B=\B_\D$ be the learner which on $m$ samples guesses randomly with probability $2c^m$ and otherwise invokes~$\A$. Clearly, \unsddelete{$\B$ correctly guesses any label with probability at least $2 c^m / 2 = c^m$, so} \unsdreplace{$\eps_{\B}(\D',m) \leq 1-c^m$}{$\eps_{\B}(\D',m) \leq 1-\frac{2c^m}{2} =1-c^m$} for any distribution $\D'$. Consider the following certifier $\C=\C_\B$: Given a sample $S$ of size $m$, if $S \sse \supp(\D)$ then output $\epsilon_{\B}(\D,m) = (1-2c^m)\epsilon_\A(\D,m) + c^m$, otherwise output $1$. This is sound for $\B$: If $S \sim \D^m$ then $\expects{}{\C(S)} = \epsilon_{\B}(\D,m)$, and if $S\sim \D'^m$ for $\D' \in \DD$ not equal to $\D$ then $\expects{}{\C(S)} >  1-c^m \geq \eps_{\B}(\D',m)$ by property (ii). This witnesses a certifiable error rate of $\epsilon_{\B}(\D,m) \leq \epsilon_{\A}(\D,m) + c^m$ for $\D$\unsddelete{, which tends to the optimal distribution-fixed error $\epsilon_{\A}(\D,m)$  for large $m$}. Since there is no smart learner over $\DD$ by property (i), whereas certifiable errors tend to distribution-fixed errors additively as $m$ grows large, there can be no relatively smart learner over~$\DD$.

It remains to exhibit a countable family on a countable domain satisfying (i) and (ii). The family of uniform distributions $\bigcup_{n,\beta} \{\D_S : S \in \S(n,\beta)\}$ from Theorem~\ref{thm:no_learner_m2_smart} is such a family. Indeed, (i) is because for $S \in \S(n,\beta)$ there is a distribution-fixed learner for $\D_S$ with error on the order of $n^{-O(\beta)}$ with $n^{O(\beta)}$ samples,\footnote{The learner which validates between $h_S$ and the constant $1$ predictor has error $\tilde{O}\left(n^{-O(\beta)} + \frac{1}{m}\right)$ on $m$ samples.} whereas absent knowledge of $S$ no nontrivial learning is possible until the number of samples exceeds $n^{1-O(\beta)}$ (Lemma~\ref{lemma:any-learner-average-error}). As for (ii), it follows by construction since sets in $\S(n,\beta)$ have vanishingly small pairwise intersections (Lemma~\ref{lemma:set_system}). 
\end{proof}

Finally, we reflect on how the difficulty of relatively smart learning varies as a function of the distribution family. Growing the family can certainly make relatively smart learning harder, as it always does for PAC learning, smart learning, and most other learning paradigms. Somewhat unusually, however, we show that the opposite can also occur for relatively smart learning. This non-monotonicity phenomenon is summarized in the following corollary of Theorems~\ref{thm:dfamily_nosmart} and~\ref{thm:oig_square_smart}. 
\begin{corollary}\label{cor:nonmonotone}
There is a hypothesis class on a countable domain, and three distribution families $\DD_1 \subset \DD_2 \subset \DD_3$, such that $\DD_1$ and $\DD_3$ admit a relatively smart learner, but $\DD_2$ does not.  
\end{corollary}
\begin{proof}
Let $\DD_2$ be the distribution family from Theorem~\ref{thm:dfamily_nosmart}, and fix its associated domain  and hypothesis class. Let $\DD_3$ consist of all distributions on the domain, and let $\DD_1= \{\D\}$ for any single distribution $\D \in \DD_2$. It is clear that any singleton family such as $\DD_1$ is smartly learnable, whereas Theorem~\ref{thm:oig_square_smart} yields a relatively-smart learner for $\DD_3$. 
\end{proof}
Corollary~\ref{cor:nonmonotone} might appear paradoxical until we recall that our benchmark in relatively smart learning---unlike most other learning paradigms---depends on the distribution family as a whole. While our desired learner may be burdened by having to handle more distributions, so too are the certifier/learner pairs with which it must compete. Specifically, for any learner $\A$ catered to a marginal $\D$, if we grow the distribution family then a certifier $\C$ for $\A$ must now be sound with respect to a larger class of marginal distributions, pushing the certifiable error on $\D$ upwards!

\printbibliography


\appendix
\section{Definitions: ERM and OIG}
\label{app:oig_erm_def}
We use the standard definition of empirical risk minimization.
\begin{definition} Empirical Risk Minimization (ERM)  over a hypothesis class $\H \sse \{0,1\}^\X$ selects as its predictor an arbitrary $h \in \H$ minimizing average error on the training data. Since we are in the realizable setting, this is an arbitrary hypothesis consistent with all the training examples. As usual, when evaluating the error of ERM we assume worst-case tie-breaking among consistent hypotheses.\end{definition}

The One-inclusion Graph (OIG) learner is best described in the context of transductive learning, as originally employed by \cite{haussler_predicting_1994}. We also find the exposition in \citep{daniely_optimal_2014,asilis_regularization_2024} helpful.

\begin{definition}[Transductive Learning]
\label{def:transductive}
In the transductive model of learning for a hypothesis class $\H \subseteq \{0,1\}^\X$, an adversary selects a collection of unlabeled instances $ S = (x_1,\ldots, x_n) \in \X^n$ and a labeling hypothesis $h\in \H$. Let $S_h =((x_1,h(x_1)), \ldots, (x_n,h(x_n))$ be the labeled sequence. Then, a data point $x_i$ is selected uniformly at random from $S$ as the test point, and the remaining data points and their labels, denoted by $S_h^{(-i)} = \{(x_j,y_j)\}_{j \neq i}$, are revealed to the learner $\A$. In other words, the learner is trained on $S_h^{(-i)}$ and tested on the label of $x_{i}$.

The transductive error rate of a learner $\A$ is defined as
\[
\epsilon_{\A}^{\trans}(S_h) = \frac{1}{n} \sum_{i=1}^n \id[\A(S_h^{(-i)})(x_i) \neq h(x_i)].
\]
Whenever the hypothesis class $\H$ is clear from context, we will overload notation and denote the worst-case transductive error of $\A$ on the unlabeled set $S$ over any $h\in \H$ as 
\[
\epsilon_{\A}^{\trans}(S) =\max_{h \in \H}\epsilon_{\A}^{\trans}(S_h).
\]
Note that any learner in the transductive model is well defined when realizable data is drawn i.i.d.~(as in the PAC model) by taking $S$ to be the union of the training and test points, then predicting accordingly. A leave-one-out argument upper bounds expected error over the distribution---where expectation is over i.i.d.~draws of training and test data---by the worst-case transductive error. 
\end{definition}

\begin{definition}[OIG Learner] 
\label{def:oig}
The One-Inclusion Graph (OIG) learner for a hypothesis class $\H \sse \{0,1\}^\X$ is that which minimizes worst-case transductive error $\epsilon_{\A}^{\trans}(S)$, simultaneously for every unlabeled dataset $S \in \X^*$.  Operationally, this learner can be described as follows for a given unlabeled dataset $S=(x_1,\ldots,x_n)$: Let $G$ be the subgraph of the hypercube $\{0,1\}^n$ induced by the behaviors $\H|S$ of $\H$ on $S$, and find the (randomized) orientation of the edges of $G$ which minimizes the maximum (expected) outdegree. When only a single behavior in $\H|S$ is consistent with the training data $\{(x_j,y_j)\}_{j \neq i}$, the algorithm predicts $y_i$ accordingly. When both labels are possible for $y_i$ given the training labels, then the oriented edge between the two corresponding behaviors determines the label chosen by the algorithm; in particular, the label $y_i$ corresponding to the behavior pointed to by the directed edge is chosen.  
\end{definition}
    
\begin{remark}\label{remark:oig_fractional}
Note that we allow the OIG algorithm to randomize over orientations. This can be equivalently described as a fractional orientation minimizing the maximum fractional out-degree, which is then subjected to independent randomized rounding. The distinction between integral (i.e. deterministic) and fractional (i.e. randomized) orientations is not particularly consequential, as the associated integrality gap---and therefore the multiplicative gap in error---is merely a factor of~2. However, the latter does have the distinction of obtaining the optimal transductive error rate. These subtleties are discussed in \cite{asilis_regularization_2024}.
\end{remark}

\section{Uniformity Testing Preliminaries}
\label{app:uniformity_testing}

In this section, we review the background of uniformity testing. There is a substantial body of work on uniformity testing, with many algorithms developed for uniformity and identity testing under a variety of modeling assumptions~\citep{GoldreichRon2000Expansion,paninski2008coincidence,diakonikolas2014testing,acharya2015optimal,valiant2017automatic, diakonikolas2018sample,diakonikolas2019collision,diakonikolas2019communication,acharya2020inference}. In this work, we adopt the collision-based approach of \citet{GoldreichRon2000Expansion}, which we refer to as the standard uniformity tester.

We first recall the definition of total variation distance for countably-supported distributions.
\begin{definition}[Total Variation distance]
Let $\D$ and $\Q$ be two probability distributions defined over a countable domain $\X$. The TV distance between $\D$ and $\Q$ is defined as \[\dtv(\D,\Q) = \frac{1}{2} \|\D - \Q\|_1 = \frac{1}{2} \sum_{x \in \X}|\D[x] - \Q[x]|.\]
\end{definition}

We now describe the standard collision-based uniformity tester of \citet{GoldreichRon2000Expansion}, denoted by $\tunifstandard$. This tester assumes that the unknown distribution $\D$ is supported on the same finite domain as the reference uniform distribution. This assumption about the support is crucial for relating collision statistics to the distance from uniformity.
The formal description of $\tunifstandard$ appears in Algorithm~\ref{alg:uniformity_test_standard} and its performance guarantee is stated below.

\begin{lemma}[\cite{diakonikolas2019collision}]\label{lemma:unif_tester_standard}
Denote $m_{\mtest}(n,\xi,\delta):= \frac{18\cdot 64\sqrt{n}\ln(2/\delta)}{\xi^2}$. For any set $Y$ of size $|Y| = n$, $\xi,\delta \in (0,1)$, any distribution $\D$ such that $\supp(\D)\subseteq Y$, and any sample size $m\geq m_{\mtest}(n,\xi,\delta)$ Algorithm~\ref{alg:uniformity_test_standard} satisfies that 
\begin{itemize}
        \item If $\D = \D_Y$, 
        $\tunifstandard_{Y}(\xi,\delta,S)$ returns $1$ with probability at least $1-\delta$ over $S \sim \D^m$
        \item If $\dtv(\D,\D_Y) > \xi$,  $\tunifstandard_{Y}(\xi,\delta,S)$  returns $0$ with probability at least $1-\delta$ over $S \sim \D^m$
    \end{itemize}
\end{lemma}

\begin{algobox}{Standard Uniformity Tester $\tunifstandard_Y(\xi,\delta,S)$ for a set $Y$ with input parameters $\xi,\delta\in(0,1)$ and sample $S$ \citep{GoldreichRon2000Expansion}.}\label{alg:uniformity_test_standard}
\begin{enumerate}[itemsep=0pt]
    \item Let $n=|Y|$, $\ell=18\ln(2/\delta)$, and $m'= |S|/\ell$.
    \item Divide $S$ into $\ell$ consecutive sub-samples $S_i= (x_{i1},\ldots,x_{im'}), 1 \leq i \leq \ell$
    \item Let $\textsc{tr} = \frac{1+2\xi^2}{n}$.
    \item For $1 \leq i \leq \ell$ do:
    \begin{enumerate}
         \item Let $Z_i = \frac{1}{{m' \choose 2}} |\{(j,k):j<k, x_{ij}=x_{ik}\}|$.
        \item If $Z_i <\textsc{tr}$ then let $\textsc{acc}_i = 1$, else let $\textsc{acc}_i = 0$.
    \end{enumerate}
    \item If $\sum_{i}\textsc{acc}_i \geq \ell/2$ then return $1$, else  return $0$. 
\end{enumerate}
\end{algobox}

In our setting, we will consider distributions that may place nonzero mass outside the reference domain $Y$. To handle this, we rely on a simple but useful observation showing how total variation distance behaves under conditioning.

\begin{lemma}\label{lemma:tv_conditional}
    Let $\X$ be a countable set and let $Y \subseteq \X$ be finite. Let $\D,\D'$ be any pair of distributions supported on $\X$.
    We have that 
    \[\dtv(\D_{|Y},\D') \geq  \dtv(\D,\D') - \D[\X \setminus Y].\]
\end{lemma}
\begin{proof}
This follows by the triangle inequality and the definition of  total variation distance.
\[ \dtv (\D',\D) \leq   \dtv(\D', \D_{|Y}) + \dtv(\D_{|Y},\D) = \dtv(\D',\D_{|Y}) + \D[\X\setminus Y] \] 
\end{proof}
Lemma~\ref{lemma:tv_conditional} shows that if a distribution $\D$ places only a small amount of probability mass outside $Y$, then conditioning $\D$ on $Y$ preserves total variation distance up to a small additive difference. Consequently, testing $\D$ for uniformity over $Y$ can be reduced to testing the conditional distribution $\D_{|Y}$, provided that samples falling outside $Y$ are explicitly detected. This observation motivates a simple wrapper around the standard uniformity tester, which rejects whenever a sample lies outside $Y$ and otherwise invokes 
$\tunifstandard$ (with modified parameters) on the sample. The modified uniformity tester $\tunif$ appears in Algorithm~\ref{alg:uniformity test}, and its guarantees are summarized in the following lemma.

\begin{lemma}\label{lemma:uniformity_tester}
    Denote $m_{\mtest}(n,\xi,\delta):= \frac{18\cdot 64\sqrt{n}\ln(2/\delta)}{\xi^2}$. Let $\X$ be an arbitrary countable set. For any set $Y \sse \X$ of size $|Y| = n$, any  $\xi,\delta \in (0,1)$, any distribution $\D$ on $\X$, and any sample size $m\geq m_{\mtest}(n,\xi/2,\delta)$, the tester $\tunif_Y(\xi,\delta,S)$ (Algorithm~\ref{alg:uniformity test}) satisfies the following:
    \begin{itemize}
        \item If $\D = \D_Y$, 
        $\tunif_Y(\xi,\delta,S)$ returns $1$ with probability $\geq 1-\delta$ over $S\sim \D^m$.
        \item If  $\dtv(\D,\D_Y) > \xi$, $\tunif_Y(\xi,\delta,S)$ returns $0$ with probability $\geq 1-\delta$ over $S\sim\D^m$.
    \end{itemize}
\end{lemma}
\begin{proof}
    Let $m \geq m_{\mtest}(n,\xi/2,\delta) = \frac{18\cdot 64\sqrt{n}\ln(2/\delta)}{(\xi/2)^2}$.  
    If $\D = \D_Y$ then observe that $S \subseteq Y$ and the first claim follows from the first guarantee of Lemma~\ref{lemma:unif_tester_standard}.
    Moreover, if $\dtv(\D,\D_Y) > \xi$ and $\D[\X \setminus Y]\leq \xi/2$, we have from Lemma~\ref{lemma:tv_conditional}  that  $\dtv(\D_{|Y},\D_Y)\geq \dtv(\D,\D_Y) - \D[\X \setminus Y] \geq \xi/2$. In this case if $S \not\subseteq Y$, then the tester returns $0$, correctly rejecting $\D$. If $S \subseteq Y$, then we know that $S$ is distributed i.i.d. according to $\D_{|Y}$. By our choice of $m$ we have with probability $1-\delta$ over $S \sim \D^m$ that the tester returns $0$ correctly rejecting.  
    Finally, if $\dtv(\D,\D_Y)> \xi$ and $\D[\X \setminus Y]> \xi/2$ then since $ m\geq \frac{2\ln(1/\delta)}{\xi}$ it is easy to check that with probability at least $1-\delta$ over $S \sim \D^m$ we have $S \not\subseteq Y$ in which case the tester returns $0$. This completes the proof. 
\end{proof}

\begin{algobox}{Modified uniformity tester $\tunif_Y(\xi,\delta,S)$ for set $Y$ with parameters $(\xi,\delta)\in(0,1)$ and input sample $S$:}\label{alg:uniformity test}
    \begin{enumerate}[itemsep=0pt]
\item If $S \subseteq Y$ then return $\tunifstandard_Y(\xi/2,\delta,S)$; {\Comment{Algorithm~\ref{alg:uniformity_test_standard} for testing against $\D_Y$}}
\item Else return 0.
\end{enumerate}
\end{algobox}

\ifthenelse{\boolean{coltsubmission}}{}{\section{Missing Proofs from Section~\ref{pf:oig_erm_fail}}}

\subsection{Proof of Lemma~\ref{lemma:certifiable_majority_error}}\label{pf:certifiable_majority_error}

We will use the following lemma as upper bound on the error of the majority learner and lower bound on the error of ERM for learning from $\H(n)$ under the uniform distribution $\D(n)$.
\begin{lemma}\label{lemma:majority_single}
    Let $\xi(n)$ be defined as in Equation~\eqref{eq:m(n)xi(n)}. Assume $n$ is sufficiently large such that $n > \frac{\log(2/\xi(n))}{2\xi(n)(1/2-\xi(n))^2}$. Then, for the hypothesis class $\HH$ and sample sizes $\frac{\log(2/\xi(n))}{2(1/2-\xi(n))^2} \leq m \leq M(n)$ we have (i) $\epsilon_{\A_{\maj}}(\D(n),m) \leq 2\xi(n)$ and (ii) $\epsilon_{\A_{\erm}}(\D(n),m) \geq 1-2\xi(n)$. 
\end{lemma}
\begin{proof}
Fix any $h\in \HH$. Let $b\in \{0,1\}$ be the minority label of $h$ on $\X_n$, i.e., $|\{x\in\X_n: h(x)=b\}| \leq M(n)$ and thus either $h\notin \H(n)$ or $h\in \H^{(b)}(n)$. Denote $\X_b:=\{x\in\X_n: h(x)=b\}$ and $\X_{1-b} = \X_n \setminus \X_b$. We know $\D(n)[\X_b] \leq M(n)/n=\xi(n)$ and $\D[\X_{1-b}] \geq 1-\xi(n)$. Let $\delta\in(0,1)$ and $S$ be any sample of size $m \geq \frac{\ln(2/\delta)}{2(1/2-\xi(n))^2}$. In the event that $|S \cap \X_{1-b}| > |S|/2$, we know $\A_{\maj}$ outputs $1-b$ on every input and clearly has error $\D(n)[\X_{b}] \leq \xi(n)$. We bound the event that the majority label in sample is not $1-b$ by Hoeffding's inequality
\[
\begin{aligned}
\probs{S \sim \D(n)^m}{|S \cap \X_{1-b}| \leq \frac{|S|}{2}} &= \probs{S \sim \D(n)^m}{\frac{1}{m} |S \cap \X_{1-b}| \leq \frac{1}{2}}\\
& \leq \probs{ S \sim \D(n)^m}{\left|(1-\xi(n)) - \frac{1}{m} |S \cap \X_{1-b}|\right| \geq\frac{1}{2} - \xi(n)} \\
& \leq 2 \exp(-2m (1/2-\xi(n))^2) \leq \delta.
\end{aligned}
\]
Therefore we can conclude that $\epsilon_{\A_{\maj}}(\D,m) \leq \xi(n) + \delta$. On the other hand if $m \leq M(n)$, there always exists an ERM hypothesis that will incorrectly pick $b$ as the majority label. This is due the fact that $|S|\leq M(n)$ and there always exists a hypothesis $\hat{h}$ consistent with $S$ that has $|\{x\in\X_n: \hat{h}(x)=1-b\}|= M(n)$. It is clear that such a hypothesis has error at least $1-2\xi(n)$. Letting $\delta=\xi(n)$ concludes the result. We also want to make sure such a value of $m$ exists. The lower bound requirement on $n$ is used to make sure that $\frac{\ln(2/\xi(n))}{2(1/2-\xi(n))^2}  \leq m < \xi (n)\cdot n = M(n)$.
\end{proof}

We now restate Lemma~\ref{lemma:certifiable_majority_error} and prove it.
\begin{lemma}[Restatement of Lemma~\ref{lemma:certifiable_majority_error}]
There exists an absolute constant $C=C(\beta)\in \NN$ such that for all $n\geq C$, the rate 
\[\epsilon_{n}(\D,m) = \begin{cases}
    4\xi(n) & \text{if $\D = \D(n)$ and $m \geq m(n)$}\\
    1 & \text{otherwise.}
\end{cases}\]
is certifiable for $\A_{\maj}$ in the distribution-free setting.
\end{lemma}
\begin{proof}
Recall the definition of functions $M,m:\NN \times \NN$ and $\xi: \NN \times (0,1)$
\[
   M(n) = \lceil n^{1-\beta}\rceil,
    m(n) = \lfloor n^{\frac{1}{2}+3\beta} \rfloor, \,\text{and}\, \xi(n) = M(n)/n \approx n^{-\beta}.
\]

We prove that $\A_{\maj}$ certifiably achieves the error rate $\epsilon_n(.,.)$ with $\C_{\maj,n}$ being the sound certifier. To do so, we have to make sure that
\[\epsilon_{\A_{\maj}}(\D,m)\leq \expects{S \sim \D^m}{\C_{\maj,n}(S)}\leq \epsilon_{n}(\D,m)\]
for all $\D,m$. 
It is obvious that whenever $m < m(n)$, regardless of the distribution, the certifier always outputs $1$, which is an upper bound on the error rate and also is equal to the certifiable error rate. We will have to consider three cases for the distribution $\D$ when $m\geq m(n)$. Throughout, we will be using the following immediate corollary of Lemma~\ref{lemma:uniformity_tester} and always assume $n$ is larger than the constant in the following to handle different cases.
\begin{lemma}\label{lemma:uniformity_tester_Xn}
    There exists an absolute constant $C_1 = C_1(\beta)\in\NN$ such that for any $n\geq C_1(\beta)$, we have for any $m\geq m(n)$,
     \begin{itemize}
        \item If $\D = \D(n)$, 
        $\tunif_{\X_n}(\xi(n),\xi(n),S)$ returns $1$ with probability at least $1-\xi(n)$ over $S \sim \D^m$.
        \item If $\dtv(\D,\D(n)) > \xi(n)$,  
        $\tunif_{\X_n}(\xi(n),\xi(n),S)$ returns $0$ with probability at least $1-\xi(n)$ over $S \sim \D^m$.
    \end{itemize}
\end{lemma}

\begin{proof}
We can verify that for sufficiently large $n$, for $m(n) = \lfloor n^{\frac{1}{2}+3\beta}\rfloor$ as defined in Equation~\eqref{eq:m(n)xi(n)} in 
   Section~\ref{pf:oig_erm_fail}  we have \begin{equation}\label{eq:m_for_test}
\lfloor n^{\frac{1}{2}+3\beta}\rfloor \geq  \frac{18\cdot 64\sqrt{n}\ln\left(2n^{\beta}\right)}{n^{-2\beta}} = \frac{18\cdot 64\sqrt{n}\ln\left(\frac{2}{\xi(n)}\right)}{\xi^2} = m_{\mtest}(n,\xi(n),\xi(n)).
\end{equation}
This implies that any $m \geq m(n)$ satisfies the requirement of Lemma~\ref{lemma:uniformity_tester} for $S = \X_n$ and $\xi = \xi(n)$, which concludes the proof.
\end{proof}

We now discuss each case separately.
\begin{enumerate}
    \item {The distribution is $\D(n)$.} 
    We can verify that for sufficiently large $n$ (and thus small $\xi(n)$), we have the following lower bounds on $n$ and $m(n)$ in terms of $\xi(n)$
\begin{equation}\label{eq:n_for_maj}
n \geq  \frac{\ln\left(2n^{\beta}\right)}{2 n^{-\beta}(\frac{1}{2}-n^{-\beta})^2} = \frac{\ln\left(\frac{2}{\xi(n)}\right)}{2\xi(n)(\frac{1}{2}-\xi(n))^2},
\end{equation}
and
\begin{equation}\label{eq:m_for_maj}
m \geq m(n) \geq  \frac{\ln\left(2n^{\beta}\right)}{2(\frac{1}{2}-n^{-\beta})^2}  = \frac{\ln\left(\frac{2}{\xi(n)}\right)}{2(\frac{1}{2}-\xi(n))^2}.
\end{equation}
Therefore, we can invoke Lemma~\ref{lemma:majority_single} to conclude that $\epsilon_{\A_{\maj}}(\D(n),m) \leq 2 \xi(n) = 2 M(n)/n$.  
Moreover, from Lemma~\ref{lemma:uniformity_tester_Xn}, we know that with probability at least $1-\xi(n)$ over $S \sim \D(n)^{m}$, $\tunif_{\X_n}(\xi(n),\xi(n),S)$ (Algorithm~\ref{alg:uniformity test}) returns $1$, accepting that the underlying distribution is $\D(n)$. Hence, with probability at least $1-\xi(n)$ the certifier outputs $3\xi(n)$. Moreover, observe that $\C_{\maj,n}(S)\geq 3\xi(n)$ for any sample $S$ and we have
\[
\epsilon_{\A_{\maj}}(\D(n),m) \leq 3\xi(n)  \leq \expects{S \sim \D(n)^{ma,}}{\C_{\maj,n}(S)} \leq 3\xi(n) + \xi(n) =4\xi(n) = \epsilon_{n}(\D(n),m).
\]

\item {\bf $\dtv(\D,\D(n)) \leq \xi(n)$}. We will prove that the error of majority learner is always bounded by $3\xi(n)$ for $n$ sufficiently large such that $\xi(n) < 1/4$. Let $b$ denote the majority label of the labeling function $h^{*}\in \HH$ under the distribution $\D$ and denote $p:=\D[\{x: h^{*}(x) =b \}]$. Considering $\dtv(\D,\D(n)) \leq \xi(n)$, we observe that for $b=1$, if $ h^{*} \in \H(n)$, we have $\D(n)[\{x\in \X_n: h^{*}(x)=1\}] = 1-\xi(n)$ and $p\geq 1-2\xi(n)$. Note that $\D(n)[\{x\in \X_n: h^{*}(x)=1\}] = 1-\xi(n)$ comes from the fact that if by the sake of contradiction, we assume for $h^{*}$ we have $\D(n)[\{x\in \X_n: h^{*}(x)=1\}] =\xi(n)$, then we would get $p <2 \xi(n)< 1/2$, which is a contradiction to $b=1$ being the majority label. If $ h^{*} \notin \H(n)$, we have $\D(n)[\{x\in \X_n: h^{*}(x)=1\}] = 1$ and $p\geq 1-\xi(n)$. On the other hand, if $b=0$, we have $h^{*} \in \H(n)$. To see this, assume by the sake of contradiction that $h^{*}\notin \H(n)$. Then we get that $\D(n)[\{x\in \X_n: h^{*}(x)=0\}] = 0$ and thus $\D[\{x\in \X_n: h^{*}(x)=0\}] \leq \xi(n)<1/2$ which is a contradiction to $0$ being the majority label. Therefore, $\D(n)[\{x\in \X_n: h^{*}(x)=0\}] = 1-\xi(n)$ and $p \geq 1-2\xi(n)$. In any case, 
we get that $1-2\xi(n) \leq p \leq 1$. Observe that $\expects{S \sim \D^{m}}{|\{(x,y)\in S: y=b\}|} \geq (1-2\xi(n))\cdot m$ and similar to proof of Lemma~\ref{lemma:majority_single} we get from Chernoff's inequality that
\[
\probs{S\sim\D^{m}}{\text{Majority of labels in $S$ is $1-b$}} \leq 2\exp\left(-2m(1/2-2\xi(n))^2\right) \leq \xi(n),
\]
where the last line follows from Equation~\eqref{eq:m_for_maj} that holds  for large $n$. 
Therefore, we get that 
\[
\epsilon_{\A_{\maj}}(\D,m) = \expects{S \sim \D^{m}}{\id[\A_{\maj}(x) \neq h^*(x)]}\leq \xi(n) p + (1-p)\leq \xi(n)  + 2\xi(n) = 3\xi(n).  
\]
Taking into account that $\C_{\maj,n}(S) \geq 3\xi(n)$ for any $S$, we have $\expects{S \sim \D^{m}}{\C_{\maj,n}(S)} \geq 3\xi(n) \geq \epsilon_{\A_{\maj}}(\D,m)$.
\item {\bf $\dtv(\D,\D(n)) > \xi(n)$.} 
  In this case, we can invoke Lemma~\ref{lemma:uniformity_tester_Xn} to get that with probability at least $1-\xi(n)$ over $S \sim \D^{m}$, $\tunif_{\X_n}(\xi(n),\xi(n),S)$ (Algorithm~\ref{alg:uniformity test}) rejects and outputs $0$. Therefore, $\C_{\maj,n}$ outputs $1$ with probability at least $1-\xi(n)$ and $\expects{S \sim \D^{m}}{\C_{\maj,n}(S)} \geq 1-\xi(n)$.

We now find an upper bound on the error of $\A_{\maj}$ on distribution $\D$.  Let $b$ denote the majority label of labeling function $h^{*}\in \HH$ under distribution $\D$ and denote  $p : = \D[\{x: h^{*}(x) =b \}] \geq 1/2$. We get from Chernoff's inequality that 
\[
q:=\probs{S\sim\D^{m}}{\text{Majority of labels in $S$ is $1-b$}} \leq 2\exp\left(-2m(p-1/2)^2\right).
\]
We can therefore write that
\[
\expects{S \sim \D_{h^{*}}^{m},x\sim \D}{\id[\A_{\maj}(x) \neq h^{*}(x)]} = p q + (1-p)(1-q).
\]
Observe that since $p\geq 1/2$, we get from simple calculations that $pq + (1-p)(1-q) \leq p$. Now if $1/2\leq p\leq 1/2+\xi(n)$ we have  $pq + (1-p)(1-q) \leq  1/2+\xi(n)$, which is at most $3/4$ for large $n$ where $\xi(n) < 1/4$. Moreover, we can verify that for sufficiently large $n$,
\[
m(n) \geq \frac{\ln(\frac{16}{1+2\xi(n)})}{2(\xi(n))^2}.
\]
This implies that if $p > 1/2+\xi(n)$ then we have $m\geq m(n) \geq \frac{\ln(\frac{16}{2p})}{2(p-1/2)^2}$ and, thus, $q \leq \frac{p}{4}$, which in turn proves $pq + (1-p)(1-q)\leq \frac{p^2}{4} + 1-p \leq \frac{1}{4} + \frac{1}{2} =\frac{3}{4}$.
Since the choice of $h^{*}\in \HH$ was arbitrary, we can get that
\[
\epsilon_{\A_{\maj}}(\D,m) = \sup_{h^{*}\in \H}\expects{S \sim \D_{h^{*}}^{m},x\sim \D}{\id[\A_{\maj}(x) \neq h^{*}(x)]} \leq \frac{3}{4}.
\]
This concludes that  
$ \expects{S \sim \D^{m}}{\C_{\maj,n}(S)} \geq 1-\xi(n) \geq 3/4 \geq  \epsilon_{\A_{\maj}}(\D,m)$. 
\end{enumerate}
\end{proof}

\subsection{Proof of Lemma~\ref{lemma:lower_bound_erm_oig}}
\label{pf:lower_bound_erm_oig}

We know from Lemma~\ref{lemma:majority_single} that for the uniform distribution $\D(n)$ and the hypothesis class $\H(n)$ there are ERM learners with error rate $\epsilon_{\A_{\erm}}(\D(n), m) \geq 1-2\xi(n)$ for all $m(n) \leq m < M(n)$ because $m < M(n)$ and that the Equations~\eqref{eq:n_for_maj} and \eqref{eq:m_for_maj} imply that requirements of the lemma are satisfied for $m \geq m(n)$.

We now consider the OIG learner. For a multiset $S\in \X^*$, denote by $\dom(S)\subseteq \X$ the set of distinct elements in $S$ and by $S_{(1)}\subseteq \dom(S)$ the set of all elements that appear exactly once in $S$. Observe that for any multiset $S \in \X_{n}^*$ with  $|S| \leq M(n)$,  the set $\dom(S)$ is shattered by $\H(n)$. From Fact~\ref{fact:out-degree-n}, we know that for any such set there exists a maximum-out-degree minimizing orientation of the graph that sets the out-degree of every vertex to at least $\floor{|S_{(1)}|/2}$. Moreover, from Lemma~\ref{lemma:single_occurrence} we know that for any $m \leq M(n)$ we have with probability at least $1-n^{-\beta/2}$ over $S \sim \D(n)^{m}$ that $|S_{(1)}|\geq m(1-2n^{-\beta/2})$. This implies that for any $h\in\H(n)$ with probability at least $1-n^{-\beta/2} = 1-o(1)$ over $S\sim \D(n)^m$, the trandsuctive error is at least $1/2 - n^{-\beta/2} = 1/2-o(1)$. Combining this with the leave-one-out argument, we conclude that $\epsilon_{\A_{\oig}}(\D(n),m) \geq 1/2 - 2n^{-\beta/2} = 1/2- o(1)$ for all $m < M(n)$.\hfill \qed

\begin{fact}\label{fact:out-degree-n}
    Let $\H \subseteq \{0,1\}^{\X}$ be a hypothesis class and $S \in \X^n$ a multiset of size $n$ such that $\dom(S)$ is shattered by $\H$, i.e., $\H{|\dom(S)} = \{0,1\}^{\dom(S)}$. There exists a maximum-out-degree minimizing orientation of the OIG on $S$ that sets the out-degree of every vertex to at least $\lfloor |S_{(1)}|/2 \rfloor$.
\end{fact}
\begin{proof}
    Let $d:=|S_{(1)}|$. It is easy to verify that the one-inclusion graph of $\H|S$ is a disjoint union of $d$-dimensional hypercubes. This is due to the fact for any $u\in\{0,1\}^n$, the edges connected to $u$ correspond exactly to the instances that appear exactly once in $S$. Moreover, any node $u$ may only connect to nodes that are consistent with $u$ on $S \setminus S_{(1)}$. It is therefore, enough to show that for the $d$-dimensional hypercube, there exists a maximum-out-degree minimizing orientation that sets the out-degree of every vertex to $\lfloor d/2 \rfloor$. We now prove this fact.
    
    First note that the total number of edges in the $d$-dimensional hypercube is $d \cdot 2^{d-1}$ and any orientation of the graph must have maximum degree at least $\lceil d/2\rceil$. We now describe an orientation that achieves maximum out-degree at most $\lceil d/2\rceil$. For any edge $e=(u,v)$ in graph connecting nodes $u,v\in\{0,1\}^d$, let $i_e\in [d]$ be the instance such that $u(i_e) \neq v(i_e)$. We orient the edge towards the node $u$ if the parity of $u$ equals $i_e$ mod $2$ and we orient it to $v$ otherwise. Obviously, each node $u$ will have out-degree at least $\lfloor d/2 \rfloor$ and at most $\lceil d/2 \rceil$. Finally, note that this is a valid orientation since $e=(u,v)$ already implies that $u$ and $v$ have different parities.
\end{proof}

\begin{lemma}\label{lemma:single_occurrence}
    Let $m \leq M(n)$. We have with probability at least $1-n^{-\beta/2} = 1- o(1)$ over $S \sim \D(n)^m$ that $|S_{(1)}|\geq m(1-2n^{-\beta/2}) = m(1-o(1))$.
\end{lemma}
\begin{proof}
    Let $Z$ denote the set of pairs of instances in the multiset $S=(x_1,\ldots,x_m)$ that are duplicates, i.e., $Z = \{1\leq i < j \leq n: x_i = x_j\}$. We can verify that $\expects{S \sim \D(n)^m}{|Z|} = {m \choose 2}\frac{1}{n}$. We apply Markov's inequality to conclude that 
    \[
    \prob{Z \geq n^{-\beta/2}m} \leq \frac{m^2}{2n}\cdot \frac{1}{mn^{-\beta/2}} \leq \frac{1}{2n^{\beta/2}},
    \]
    where the last line follows from the fact that $m \leq M(n) = n^{1-\beta}$. Note that if we remove all the instances in $Z$ from $S$, we get the set of all instances of $S$ that only appear once. Therefore, $|S_{(1)}| \geq |S| - 2|Z|$ and we get that with probability at least $1-n^{-\beta/2}$ we have $|S_{(1)}| \geq m(1 - 2 n^{-\beta/2})$. This concludes the proof.
\end{proof}

\section{Missing Proofs From Section~\ref{pf:oig_square_smart}}\label{app:OIG_smart}

\subsection{Proof of Lemma~\ref{lemma:expected_to_transductive}}
 Fix a multiset $S=(x_1,\ldots,x_M)$ of $M$ instances. 
   Let $\A_S$ be the learner with $\epsilon_{\A_S}(\D_S,M-1) = \epsilon^*(\D_S,M-1)$. We will now prove that there exists another learner $\A'$ such that \[\epsilon_{\A'}^{\trans}(S) \leq e\cdot \epsilon_{\A_S}(\D_S,M-1) = e\cdot\epsilon^*(\D_S,M-1),\] 
   On any set $S_h^{(-i)}$ and test point $x$, the learner $\A'$ will draw a set $T$ of $|S_h^{(-i)}| = M-1$ i.i.d. samples from the uniform distribution on $S_h^{(-i)}$, which we denote by $\D_{h,-i}$. It then predicts $\A'(x) = \A_S(T)(x)$.
  \citet{asilis2025proper} prove that in the process of sampling $M-1$ instances from the uniform distribution $\D_S$, the probability that a point $x\in S$, e.g., $x_i$, is not sampled is at least $1/e$.
  For any $h\in\H$, denote by $\D_h$ the uniform distribution on $S$. This implies that 
 \[
 \begin{aligned}
     \epsilon_{\A'}^{\trans}(S) =\max_{h \in \H}\epsilon_{\A'}^{\trans}(S_h) & = \max_{h \in \H} \frac{1}{M} \sum_{i=1}^M 1 \{\A'(S_h^{(-i)})(x_i) \neq h(x_i)\}\\
     & = \max_{h \in \H}\frac{1}{M} \sum_{i=1}^M  \expects{T \sim \D_{h,-i}^{M-1}}{1\{\A_S(T)(x_i) \neq h(x_i)\}}\\
     & \leq \max_{h \in \H} e \cdot \frac{1}{M} \sum_{i=1}^M  \expects{T \sim \D_{h}^{M-1}}{1\{\A_S(T)(x_i) \neq h(x_i)\}}\\
      & =  e \cdot \max_{h \in \H} \expects{T \sim \D_{h}^{M-1}}{\frac{1}{M} \sum_{i=1}^M 1\{\A_S(T)(x_i) \neq h(x_i)\}}\\
      & =  e \cdot \epsilon_{\A_{S}}(\D_S,M-1)\\
      & =  e  \cdot \epsilon^*(\D_S,M-1).\\
 \end{aligned}
\]
 
    We now rely on the fact that the one-inclusion-graph learner achieves the optimal transductive error on any sample $S$ (see Definition~\ref{def:oig} and Remark~\ref{remark:oig_fractional}), and therefore $\epsilon_{\A_{\oig}}^{\trans}(S) \leq \epsilon^{\trans}_{\A'}(S) \leq  e  \cdot \epsilon^*(\D_S,M-1)$. Moreover, observe that by a simple leave-one-out argument we have  
    \[
    \begin{aligned}
    \expects{S \sim \D^M}{\epsilon_{\A_{\oig}}^{\trans}(S)}  & = \expects{S \sim \D^M}{\max_{h \in \H}\epsilon_{\A_{\oig}}^{\trans}(S_h)} \geq  \max_{h \in \H}\expects{S \sim \D^M}{\epsilon_{\A_{\oig}}^{\trans}(S_h)} = \epsilon_{\A_{\oig}}(\D,M-1).
    \end{aligned}\]
\hfill \qed

\section{Missing Proofs from Section~\ref{sec:dfree_lb}}\label{sec:proof_failure_any_learner}

\subsection{Proof of Lemma~\ref{lemma:set_system}}\label{pf:set_system}
    We use a probabilistic method argument. Let $\U$
 be a universe of size $|\U| = \lceil n^{1+\beta}\rceil$ . Pick $k:= \lceil\exp(\frac{1}{4}n^{1-\beta/2})\rceil$ many subsets uniformly at random from all the subsets of $\U$ of size $n$. We calculate the expected size of the intersection of any such pair by writing
 \[
 \expect{|S_i \cap S_j|} = \sum_{u\in \U}\expect{1\{ u \in S_i \wedge u \in S_j\}} = \sum_{u\in \U}\left(\frac{n}{|\U|}\right)^2 = \frac{n^2}{|\U|} \in \left(\frac{5}{6}n^{1-\beta},n^{1-\beta}\right),
 \]
 where the last step follows from the fact that $n^{1+\beta} \leq |\U| \leq \frac{6}{5}n^{1+\beta}$ for $n\geq 7$.
 Since the random variables $1\{ u \in S_i \wedge u \in S_j\}$ are negatively associated we can now apply a Chernoff and union bound to conclude that
 \[
 \begin{aligned}
     \prob{\forall 1\leq i < j \leq k,\, |S_i \cap S_j| \leq n^{1-\beta/2}} 
     &\geq 1- \sum_{1\leq i < j \leq k}\prob{|S_i \cap S_j| > n^{1-\beta/2}}\\
     &\geq 1- k^2\prob{|S_i \cap S_j| > n^{1-\beta}\cdot (1 + n^{\beta/2}-1)}\\
     &\geq 1- k^2\prob{|S_i \cap S_j| > \expect{|S_i \cap S_j|}\cdot (1 + n^{\beta/2}-1)}\\
     &\geq 1- k^2 \exp\left(-\frac{(n^{\beta/2}-1)^2}{2+n^{\beta/2}-1}\cdot \frac{5}{6}\cdot n^{1-\beta}\right)\\
     &\geq 1- k^2 \exp\left(-\frac{5}{6}\cdot \frac{64}{90}\frac{(n^{\beta/2})^2}{n^{\beta/2}}n^{1-\beta}\right)\\
     &\geq 1- k^2 \exp\left(-\frac{16}{27}\cdot n^{1-\beta/2}\right)\\
\end{aligned}
 \]
 where we used the fact that for large $n$, we have both $n^{\beta/2}-1 \geq \frac{8}{9}n^{\beta/2}$ and $1 + n^{\beta/2}\leq \frac{10}{9}n^{\beta/2}$.
Now since we have $k^2 \leq 2\exp(\frac{1}{2}n^{1-\beta/2})$, we get 
 \begin{equation}\label{eq:subset}
 \begin{aligned}
     \prob{\forall i,j \in [k],\, |S_i \cap S_j| \leq n^{1-\beta/2}} \geq 1- k^2 \exp\left(-n^{1-\beta/2}\right) \geq 1- 2\exp\left(-\frac{1}{12}n^{1-\beta/2}\right).\\
\end{aligned}
 \end{equation}

We then analyze the requirement needed for property~\ref{property:container_size}. Take any subset $T \subset \U$ of size $|T| \leq 2n^{1-\beta}$. For any $i \in [k]$ we have that 
\[
\prob{ T \subset S_i } = \frac{{|\U|-|T| \choose |S_i| - |T|}}{{|\U| \choose |S_i|}} = \frac{|S_i| (|S_i| -1)\ldots (|S_i| - |T| +1)}{|\U| (|\U| -1)\dots (|\U| - |T| +1)}\geq \left(\frac{|S_i| - |T|}{|\U| - |T|}\right)^{|T|}.
\]
Therefore, noting that $|T| \leq 2n^{1-\beta}$, we can write for large $n$ that
\[
\prob{ T \subset S_i } \geq \left(\frac{n - 2n^{1-\beta}}{n^{1+\beta}}\right)^{2n^{1-\beta}} = \left(\frac{n^{\beta} - 2}{n^{2\beta} }\right)^{2n^{1-\beta}} \geq \left(\frac{1}{2n^{\beta}}\right)^{2n^{1-\beta}} = \exp\left(-2n^{1-\beta}\ln(2n^{\beta})\right).
\]
We can then conclude by linearity of expectation that
\[
\begin{aligned}
   \expect{ \sum_{i\in[k]} \id [ T \subset S_i ]} &\geq k\exp\left(-2n^{1-\beta}\ln(2n^{\beta})\right)\\
    &\geq \exp\left(\frac{1}{4}n^{1-\beta/2}-2n^{1-\beta}\ln(2n^{\beta})\right)\\
    & \geq \exp\left(n^{1-\beta/2}\left(\frac{1}{4}-2n^{-\frac{\beta}{2}}\ln(2n^{\beta})\right)\right) \\
    &\geq \exp\left(\frac{1}{8}n^{1-\beta/2}\right),
\end{aligned}
\]
where the last inequality follows from the fact that for large $n$, we have that $2n^{-\frac{\beta}{2}}\ln(2n^{\beta}) < 1/8$.
Note that from a Chernoff bound we have that
\[
\prob{\sum_{i\in[k]} \id[ T \subset S_i ] < \frac{1}{2}\exp\left(\frac{1}{8}n^{1-\beta/2}\right) } \leq \exp\left(-\frac{1}{8}\exp\left(\frac{1}{8}n^{1-\beta/2}\right)\right)
\]
We apply a union bound over all subsets of size at most $2n^{1-\beta}$ to conclude that
\[
\begin{aligned}
   &\prob{\forall T \subset \U: \sum_{i\in[k]} \id\{ T \subset S_i \} \geq \frac{1}{2}\exp\left(\frac{1}{8}n^{1-\beta/2}\right) }\\
   &= 1 - \prob{\exists T \subset \U: \sum_{i\in[k]} \id\{ T \subset S_i \} < \frac{1}{2}\exp\left(\frac{1}{8}n^{1-\beta/2}\right) }\\
   &\geq 1 - \sum_{T\subset \U: |T| \leq 2n^{1-\beta}}{\prob{\sum_{i\in[k]} \id\{ T \subset S_i \} < \frac{1}{2}\exp\left(\frac{1}{8}n^{1-\beta/2}\right) }}\\
   & \geq 1 - 2n^{1-\beta}{2n^{1+\beta} \choose 2n^{1-\beta}}{\prob{ \sum_{i\in[k]} \id\{ T \subset S_i \} < \frac{1}{2}\exp\left(\frac{1}{8}n^{1-\beta/2}\right)}}\\
 & \geq 1 - 2n^{1-\beta}\left(en^{2\beta}\right)^{2n^{1-\beta}}\exp\left(-\frac{1}{8}\exp\left(\frac{1}{8}n^{1-\beta/2}\right)\right),
\end{aligned}
\]
where we used the fact that $|\U|\leq 2n^{1+\beta}$ in the second inequality.
We can continue writing
\begin{equation}\label{eq:containment}
\begin{aligned}
 &1 - 2n^{1-\beta}\left(en^{2\beta}\right)^{2n^{1-\beta}}\exp\left(-\frac{1}{8}\exp\left(\frac{1}{8}n^{1-\beta/2}\right)\right) \\
 &=  1 - \exp\left(\ln(2)+ (1-\beta)\ln(n) + (2+4\beta\ln(n)) n^{1-\beta}\right)\exp\left(-\frac{1}{8}\exp\left(\frac{1}{8}n^{1-\beta/2}\right)\right)\\
 & \geq 1 - \exp\left( 2(2+4\beta\ln(n)) n^{1-\beta} -\frac{1}{8}\exp\left(\frac{1}{8}n^{1-\beta/2}\right)\right)\\
 & \geq 1- \exp\left(-\frac{1}{16}\exp\left(\frac{1}{8}n^{1-\beta/2}\right)\right)\\
 & \geq 1- \exp\left(-\frac{1}{2}n^{1-\beta/2}\right),
\end{aligned}
\end{equation}
where the first and second inequalities hold since for sufficiently large $n$ we have $\ln(2)+ (1-\beta)\ln(n) \leq (2+4\beta\ln(n)) n^{1-\beta}$, and $ 2(2+4\beta\ln(n)) n^{1-\beta} \leq \frac{1}{16}\exp\left(\frac{1}{8}n^{1-\beta/2}\right)$.

Taking a union bound over Equation~\eqref{eq:subset} and \eqref{eq:containment} we can conclude that with probability at least $1-3\exp\left(-\frac{1}{12}n^{1-\beta/2}\right)$ the intersection of every pair is at most $n^{1-\beta/2}$ and also every subset of universe of size at most $2n^{1-\beta}$ is contained in at least $\frac{1}{2}\exp\left(\frac{1}{8}n^{1-\beta/2}\right)$ many other sets. This proves the existence of a set system with the claimed properties.\hfill\qed

\subsection{Proof of Lemma~\ref{lemma:labeling_set_system}}\label{pf:labeling_set_system}

    We will use a probabilistic method to show such a hypothesis class exist. In particular, for every $S\in \S(n)$, pick a random hypothesis $h_S$ that is constant $1$ outside $S$ and the labeling on $S$ is chosen uniformly at random from all the $2^{|S|}$ possible labelings. Clearly this satisfies property~\ref{property:constant_outside}. Recall that $|S| = n$ and $|\U(n)| = \lceil n^{1+\beta}\rceil$. We show that a randomly chosen labeling of $h_S$ for all $S\in \S(n)$ satisfies the claimed properties with non-zero probability.

    To prove property~\ref{property:balanced_container} we first show that with high probability over the random labeling, all the hypotheses $h_S,S\in \S(n)$ are unique, that is, for any distinct pair of $S,S'\in \S(n)$ we have $h_S\neq h_S'$. Since the hypothesis $h_S$ is constant $1$ on any $S'\setminus S$, it is sufficient to prove that for all $S\in \S(n)$, we have $|\{x\in S: h_S(x)=0\}|\geq n/3$ because $ S\cap S'\leq n^{1-\beta/2}$. We know for any $S$ that $\expect{|\{x\in S: h_S(x)=0\}|}=n/2$. Taking a Chernoff and union bound we get that
    \begin{equation}\label{eq:all_h_unique}
    \begin{aligned}
      \prob{\forall S\in \S(n), |\{x\in S: h_S(x)=0\}|\geq \frac{n}{3}}&\geq 1- \sum_{S\in \S(n)}\prob{ |\{x\in S: h_S(x)=0\}|< \frac{n}{3}}\\
      &\geq 1-|\S(n)| \exp\left( -\frac{n}{36}\right)\\
      &\geq 1- \exp\left(\frac{1}{2}n^{1-\beta/2} - \frac{n}{36}\right)\\
      &\geq 1- \exp\left(-\frac{1}{2}n^{1-\beta/2}\right),
    \end{aligned}
  \end{equation}
   where we used the fact that for large $n$, we have $n/36\geq n^{1-\beta/2}$. 
  
    We turn to proving property~\ref{property:balanced_container}. For any $T \subset \U(n) $ with $|T| \leq 2n^{1-\beta}$, the set is contained in $\S_T$ many sets and the labeling of each set is picked uniformly at random. For any fixed labeling $b\in \{0,1\}^{|T|}$, let $Z_{T,b}$ denote the random variable $|\{ S \in \S_T, h_S|T = b\}|$. We know from property~\ref{property:container_size} in Lemma~\ref{lemma:set_system} that $\S_T \geq \lfloor \frac{1}{2}\exp\left(\frac{1}{8}n^{1-\beta/2}\right)\rfloor \geq \frac{1}{4}\exp\left(\frac{1}{8}n^{1-\beta/2}\right)$ and, therefore, we have $\expect{Z_{T,b}} = \frac{|\S_T|}{2^{|T|}} \geq 2^{-2n^{1-\beta}}\frac{1}{4}\exp\left(\frac{1}{8}n^{1-\beta/2}\right)$ . We apply a Chernoff bound to conclude that
    \[
    \begin{aligned}
    \prob{\left|Z_{T,b} - \frac{|\S_T|}{2^{|T|}}\right| \geq n^{-\beta}\frac{|\S_T|}{2^{|T|}}} & \leq 2\exp\left(-\frac{|\S_T|}{3\cdot 2^{|T|}}n^{-2\beta}\right) \\
    &\leq 2\exp\left(-\frac{n^{-2\beta}}{24}\exp\left(\frac{1}{8}n^{1-\beta/2} - 2\ln(2)n^{1-\beta}\right)\right)\\
     & = 2\exp\left(-\frac{1}{24}\exp\left(\frac{1}{8}n^{1-\beta/2} -2\beta\ln(n) - 2\ln(2)n^{1-\beta}\right)\right)\\
      &\leq 2\exp\left(-\frac{1}{24}\exp\left(\frac{1}{16}n^{1-\beta/2}\right)\right),
    \end{aligned}
    \]
    where the last inequality is due to the fact that $2\beta\ln(n) + 2\ln(2)n^{1-\beta} \leq \frac{1}{16}n^{1-\beta/2} $ for sufficiently large $n$.
    
    Taking a union bound over all labelings $b$ and all sets $T$, we get that
  \begin{equation}\label{eq:set_system_hypo_zero_one}
    \begin{aligned}
        &\prob{\text{$\exists T,b: \left|Z_{T,b} - \frac{|\S_T|}{2^{|T|}}\right| \geq n^{-\beta}\frac{|\S_T|}{2^{|T|}}$}} \leq \sum_{\substack{T\subset \U(n), b\in\{0,1\}^{|T|}: \\|T|\leq 2n^{1-\beta}}} \prob{\left|Z_{T,b} - \frac{|\S_T|}{2^{|T|}}\right| \geq n^{-\beta}\frac{|\S_T|}{2^{|T|}}}\\
        & \leq 4 n^{1-\beta}{2n^{1+\beta} \choose 2n^{1-\beta}} 2^{2n^{1-\beta}}\exp\left(-\frac{1}{24}\exp\left(\frac{1}{16}n^{1-\beta/2}\right)\right)\\ 
        & \leq 4 n^{1-\beta}\left(en^{2\beta}\right)^{2n^{1-\beta}} 2^{2n^{1-\beta}}\exp\left(-\frac{1}{24}\exp\left(\frac{1}{16}n^{1-\beta/2}\right)\right)\\ 
        & \leq \exp\left(\ln(4)+ (1-\beta)\ln(n) + 2(1+\ln(2)+2\beta\ln(n)) n^{1-\beta}\right)\exp\left(-\frac{1}{24}\exp\left(\frac{1}{16}n^{1-\beta/2}\right)\right)\\ 
         & \leq \exp\left(-\frac{1}{48}\exp\left(\frac{1}{16}n^{1-\beta/2}\right)\right),
    \end{aligned}
    \end{equation}
    where in the last inequality we used the fact that for large sufficiently large $n$, we have 
    \[
    \ln(4)+ (1-\beta)\ln(n) + 2(1+\ln(2)+2\beta\ln(n)) n^{1-\beta} \leq \frac{1}{48}\exp\left(\frac{1}{16}n^{1-\beta/2}\right)
    \]
    Taking another union bound with Equation~\eqref{eq:all_h_unique}, we conclude that with probability at least $1-\exp\left(-\frac{1}{48}\exp\left(\frac{1}{16}n^{1-\beta/2}\right)\right)-\exp\left(-\frac{1}{2}n^{1-\beta/2}\right)$, all $h_S$ are unique and, therefore, $Z_{T,b}=|\{h_S \in \H(n): S \in \S_T, h_S|T = b\}|$ for all $T,b$ and property~\ref{property:balanced_container} is satisfied for all subsets of size at most $2n^{1-\beta}$. Observe that for the large $n$ regime we are considering, property~\ref{property:balanced_container} also implies that every labeling of $T$ induces  many hypotheses and that $T$ is shattered. This implies the existence of the claimed labeling of the set-system and finishes the proof. \hfill\qed
\subsection{Proof of Lemma~\ref{lemma:cert_error_rate_no_learner_smart}}\label{pf:cert_error_rate_no_learner_smart}

We will prove that the certifier is both sound and upper bounded by the certifiable error rate. Formally, we prove that
\[\epsilon_{\A_{S}}(\D,m)\leq \expects{T \sim \D^m}{\C_{S}(T)}\leq \epsilon_{n}(\D,m).\]
The claim is obvious when $m < m(n)$ where we always have $\epsilon_{\A_S} (\D,m) \leq  \expects{T \sim \D^m}{\C_{S}(T)}= \epsilon_{n}(\D,m) = 1$.

Let $h^*\in\HH$ be the labeling function. We will use the following throughout the proof to bound the error of the learner $\A_S$. For large $n$ and $m \geq m(n)$, using a Chernoff and union bound  we have 
\begin{equation}\label{eq:h1_h0_concentration}
\probs{T_2 \sim \D_{h^*}^{m/2}}{\exists h \in \{h_S,\A_{\maj}(T_1)\},\, \left|L(h,T_2) - L(h,\D_{h^{*}})\right| > \xi(n)} \leq 4\exp\left(-m\xi(n)^2\right) \leq \xi(n),
\end{equation}
where the last inequality is due to the fact that for large $m$, we have  $m \geq \frac{\ln(4/\xi(n))}{\xi(n)^2}$.

We now consider three cases based on the underlying distribution. We will use the following instantiation of Lemma~\ref{lemma:uniformity_tester} which can be proven similar to Lemma~\ref{lemma:uniformity_tester_Xn} since Equation~\eqref{eq:m_for_test} still holds in this setting. We will assume that $n$ is larger than the constant in the following lemma for the rest of the proof.

\begin{lemma}\label{lemma:uniformity_test_S}
    There exists an absolute constant $C_1 = C_1(\beta)\in\NN$ such that for any $n\geq C_1(\beta)$, we have for any $S\in \S(n)$ and $m\geq m(n)$,
     \begin{itemize}
        \item If $\D = \D_S$, 
        $\tunif_{S}(\xi(n),\xi(n),T)$ returns $1$ with probability at least $1-\xi(n)$ over $T \sim \D^m$
        \item If $\dtv(\D,\D_S) > \xi(n)$,  
        $\tunif_{S}(\xi(n),\xi(n),T)$ returns $0$ with probability at least $1-\xi(n)$ over $T \sim \D^m$
    \end{itemize}
\end{lemma}

Let $r := |\{ (x,y) \in T_1: y=1\}|$ be the number of instances in $T_1$ with label $1$. We now discuss each case for the distribution separately.
\begin{enumerate}

\item $\D = \D_S$ for some $S\in\S(n)$.
From Lemma~\ref{lemma:uniformity_test_S} we know with probability at least $1-\xi(n)$ over $T \sim \D^{m}$, $\tunif_{S}(\xi(n),\xi(n),T)$ accepts the underlying distribution as $\D_S$ and thus $O = 1$. Therefore, we have $\expects{T \sim \D^{m}}{\C_S(T)} \leq \xi(n) + 6\xi(n) = 7\xi(n)$. 

We now bound $\epsilon_{\A_S}(\D,m)$ for $m \geq m(n)$. Two possibilities can happen, either (1) $h^* = h_S$ or (2) $h^* = h_{S'}$ for some $S'\in \S, S' \neq S$. In the first possibility, $h_S$ will have zero error under $\D_S$ and obviously $L(h_S,T_2) = L(h_S,\D_{h^*})  = 0$.

In the second possibility where $h^{*} = h_{S'}$ for some $S' \neq S$ we know from property~\ref{property:constant_outside} in Lemma~\ref{lemma:labeling_set_system} that $h_{S'}$ is $1$ on $S$ except on the intersection of $S$ and $S'$ which has size at most $n^{1-\beta/2}$ based on property~\ref{property:intersection_size} in Lemma~\ref{lemma:set_system}. Therefore, $\expects{T_1}{r} \geq \frac{|S| - |S\cap S'|}{|S|}\cdot |T_1| \geq (1-n^{-\beta/2})|T_1|$. Therefore, from a Chernoff bound, for sufficiently large $n$, we get that
\[
\begin{aligned}
    \probs{T}{r \leq \frac{1}{2}\cdot |T_1|} &= \probs{T}{ r \leq  \frac{1}{2(1-n^{-\beta/2})}\cdot (1-n^{-\beta/2})|T_1|} \\
    & \leq  \probs{T}{ r \leq  (1 -\frac{1}{3})\cdot (1-n^{-\beta/2})|T_1|} && \left( \frac{1}{2(1-n^{-\beta/2})} \leq \frac{2}{3}\right)\\
    &\leq \exp\left(-\frac{1}{18}(1-n^{-\beta/2})|T_1|\right)\leq \exp\left(-\frac{m}{18\cdot 2 \cdot 4}\right) && \left( (1-n^{-\beta/2}) \geq \frac{1}{2}\right),
\end{aligned}
\]
where we used the fact that $|T_1| \geq \lfloor \frac{|T|}{2}\rfloor \geq \frac{m}{4}$. It is easy to verify in the large $n$ and $m\geq m(n)$ regime we have $\exp\left(-\frac{m}{144}\right) \leq n^{-\beta/2}\leq \xi(n)$. Therefore, with probability at least $1-\xi(n)$ we get that $r > \frac{1}{2}\cdot |T_1|$ and $\A_{\maj}(T_1)(x) = 1$ for all $x\in \U(n)$ and thus $L(\A_{\maj}(T_1),\D_{h^{*}}) \leq \frac{|S\cap S'|}{|S|}\leq n^{-\beta/2}\leq\xi(n)$.

Let $\hat{h} \in \arg\min_{h \in \{h_S,\A_{\maj}(T_1)\}}L(h,\D_{h^*})$. We proved in above that  $L(\hat{h},\D_{h^{*}})\leq \xi(n)$ with probability at least $1-\xi(n)$.  Noting that $\A_S(T) \in \arg\min_{h \in \{h_S,\A_{\maj}(T_1)\}}L(h,T_2)$ and taking a union bound of the above failure with Equation~\eqref{eq:h1_h0_concentration}, we get that with probability at least $1-2\xi(n)$ over $T \sim \D^m$
\[
L(\A_S(T),\D_{h^{*}}) \leq L(\A_S(T),T_2) + \xi(n) \leq L(\hat{h},T_2) +\xi(n) \leq L(\hat{h},\D_{h^{*}}) + 2\xi(n)\leq 3\xi(n).
\]
Therefore, we get that 
\[
\epsilon_{\A_S}(\D_S,m) = \sup_{h^{*} \in \H} \expects{T \sim \D_{h^*}^m}{L(\A_S(T),\D_{h^{*}})} \leq 5\xi(n) \leq  \expects{T \sim \D^m}{\C_{S}(T)} \leq 7\xi(n) = \epsilon_n(\D_S,m).
\]

\item $\dtv(\D,\D_S) \leq \xi(n)$.
In this case, we have no guarantees on the success of the uniformity tester. Nevertheless, we show that the error of the certifier can bound the error of the learner. We consider the two possibilities, namely, (1) $h^{*} = h_S$ and (2) $h^{*} = h_{S'}$ for $S'\in \S, S' \neq S$.

In the first possibility, where $h^*=h_S$, we know $h_S$ will have zero error under $\D_{h^*}$ and obviously $L(h_S,T_2) = L(h_S,\D_{h^*})  = 0$.

Denote $p_1:=\D[\{u \in \U(n): h^{*}(u) = 1\}]$ and note that because $\dtv(\D,\D_S) \leq \xi(n)$ we have 
\[ 
|\D_S[\{u \in \U(n): h^{*}(u) = 1\}] - p_1| \leq \xi(n).
\]
In the case where $h^{*} = h_{S'}$ for some $S' \neq S$, we know $\D_S[\{u \in \U(n): h^{*}(u) = 1\}] \geq  1-n^{-\beta/2}$ and therefore $\expects{T_1}{r} = p_1 \cdot |T_1| \geq (1-2n^{-\beta/2})|T_1|$. A Chernoff bound, for sufficiently large $n$, concludes that
\[
\begin{aligned}
    \probs{T}{r \leq \frac{1}{2}\cdot |T_1|} & =  \probs{T}{ r \leq  \frac{1}{2(1-2n^{-\beta/2})}\cdot (1-2n^{-\beta/2})|T_1|} \\
    & \leq  \probs{T}{ r \leq  (1 -\frac{1}{3})\cdot (1-2n^{-\beta/2})|T_1|} && \left( \frac{1}{2(1-2n^{-\beta/2})} \leq \frac{2}{3}\right)\\
    &\leq \exp\left(-\frac{1}{18}(1-2n^{-\beta/2})|T_1|\right) \\
    &\leq \exp\left(-\frac{m}{144}\right)  \leq \xi(n) && \left( (1-2n^{-\beta/2}) \geq \frac{1}{2}\right).
\end{aligned}
\]
This implies that with probability at least $1-\xi(n)$, we have $r > |T_1|/2$ and $\A_{\maj}(T_1)(x) =1$ for all $x$. Therefore, $L(\A_{\maj}(T_1),\D_{h^{*}}) \leq 1 - \D[S \setminus S']\leq 2\xi(n)$, where the first inequality is due to majority being correct on $S \setminus S'$ and the second inequality is due to $\D[S \setminus S'] \geq \D_S[S \setminus S'] -\xi(n) \geq 1-2\xi(n)$. Again, we proved that with probability at least $1-\xi(n)$ we have $L(\hat{h},\D_{h^{*}})\leq 2\xi(n)$ for  $\hat{h} \in \arg\min_{h \in \{h_S,\A_{\maj}(T_1)\}}L(h,\D_{h^*})$. Similar to the previous case, we can take a union bound over the above and the failure of Equation~\eqref{eq:h1_h0_concentration} to conclude that with probability at least $1-2\xi(n)$ over $T \sim \D^m$ we have
$L(\A_S(T),\D_{h^{*}}) \leq L(\hat{h},\D_{h^{*}}) + 2\xi(n) \leq 4\xi(n)$. This proves that
\[
\epsilon_{\A_S}(\D,m) = \sup_{h^{*} \in \H} \expects{T \sim \D_{h^*}^m}{L(\A_S(T),\D_{h^{*}})} \leq 6\xi(n) \leq  \expects{T \sim \D^m}{\C_{S}(T)} \leq \epsilon_n(\D,m) =1.
\]

\item $\dtv(\D,\D_S) > \xi(n)$.
In this case we know that with probability at least $1-\xi(n)$ over $T\sim \D^m$ the test $\tunif_S(\xi(n),\xi(n),T)$ rejects $\D$ and outputs $0$. Therefore, $\expects{T \sim \D^m}{\C_S(T)} \geq 1-\xi(n)$.

Similar to the reasoning in the Case 3 of the proof of Lemma~\ref{lemma:certifiable_majority_error}, we can conclude that for any distribution $\D$, the error of the majority learner is always upper bounded by $3/4$ for large $n$. We get from Equation~\eqref{eq:h1_h0_concentration} that with probability at least $1-\xi(n)$ over $T \sim \D^m$,
\[
L(\A_S(T),\D_{h^{*}}) \leq L(\A_S(T),T_2) + \xi(n) \leq L(\A_{\maj}(T_1),T_2) +\xi(n) \leq L(\A_{\maj}(T_1),\D_{h^{*}}) + 2\xi(n).
\]
This concludes that
\[
\epsilon_{\A_S}(\D,m) = \sup_{h^{*} \in \H} \expects{T \sim \D_{h^*}^m}{L(\A_S(T),\D_{h^{*}})} \leq \frac{3}{4} + 3\xi(n)\leq \frac{7}{8} \leq \expects{T \sim \D^m}{\C_{S}(T)} \leq \epsilon_n(\D,m),
\]
where we used the fact that $\xi(n)\leq 1/24$ for large $n$. 

\end{enumerate}

Overall, we proved that $\C_S$ is a sound certifier for $\A_S$ and together the collections $\A_S$ and $\C_S$ for $S\in\S(n)$ witness the certifiable error rate $\epsilon_n(.,.)$.\hfill \qed

\newcommand{\eplaceholder}{\textcolor{red}{n^{-2\beta}}}
\subsection{Proof of Lemma~\ref{lemma:any-learner-average-error}}\label{pf:any-learner-average-error}
{\bf Notations.} For any function $h: \U \rightarrow \{0,1\}$, any $(x,y)\in \U \times \{0,1\}$, define $L(h,(x,y)): = \id [h(x) \neq y]$. Moreover, for any unlabeled set $W \in \U^*$ and labeling function $h^{*}$, define $L(h,W, {h^{*}}): = \frac{1}{|W|} \sum_{x\in W} L(h,(x,h^*(x)))$. For distribution $\D$ over $\U$, and labeling function $h^*$, define $L(h,\D, {h^{*}}): = \expects{x\sim \D}{L(h,(x,h^*(x)))}$. For any multiset $W\in \U^*$, we denote by $\supp(W)\subseteq \U$ the set of all distinct elements in $W$. For any multiset $W\in \U^*$ and test point $x$, let $W_x := W \cup \{x\}$. For any multiset $T\in (\U(n)\times\{0,1\})^*$, we denote by $\dom(T)\in\U(n)^*$ the unlabeled part of $T$. Finally, recall that for any multiset $W \in \U^*$, we denote $\S_{\supp(W)} = \{S \in \S(n): \supp(W) \subseteq S\}$. 
\begin{proof}
 
    Let $\DD_{n}$ be the uniform distribution over $\{\tilde{\D}_S: S \in \S(n)\}$ where $\tilde{\D}_S:= (\D_S,h_S)$ defines a learning instance with $\D_S$, the uniform distribution on $S$, being the marginal distribution and $h_S$ being the labeling function. We will show that a randomly picked distribution from $\DD_{n}$ is expected to incur high error on $\A$. Formally, we will prove that 
    \begin{equation}\label{eq:error_expected}
    \expected{\tilde{\D}_S\sim \DD_{n}}{\expects{T \sim \tilde{\D}_S^m}{L(\A(T),\D_S,h_S)}} = \expected{\tilde{\D}_S \sim \DD_{n}}{\expects{T \sim \tilde{\D}_S^m}{L(\A(T),S,h_S)}} \geq \frac{1}{2}-2n^{-\beta},
    \end{equation}
    this will be enough to show the existence of the claimed set $S^*_{n,m}$. 
    
    For any multiset $T$, denote $\overline{T}:= \dom(T)$. Observe that drawing $T \sim \tilde{\D}_S^m$ is equivalent to drawing the unlabeled multiset $\overline{T} \sim \D_S^m$ and then labeling it with $h_S$. Moreover, for any $S \in \S(n)$ and any $W\in S^m$ we have $\D_S^m[W] = |S|^{-m} = n^{-m}$. Therefore, we can write that 
\begin{equation}\label{eq:error_uniform_over_uniforms}
\begin{aligned}
   \expected{\tilde{\D}_S\sim \DD_{n}}{\expects{T \sim \tilde{\D}_S^m}{L(\A(T),\D_S,h_S)}} &= \expected{\tilde{\D}_S \sim \DD_{n}}{\expects{T \sim \tilde{\D}_S^m}{\frac{1}{n}\sum_{x\in S}L(\A(T),(x,h_S(x)))}} \\
   & = \frac{1}{|\S(n)|}\sum_{S \in \S(n)}\left[\frac{1}{n^m}\sum_{\overline{T} \in S^m}\left[\frac{1}{n}\sum_{\substack{x\in S}}L(\A(T),(x,h_S(x)))\right]\right].
\end{aligned}
    \end{equation}

    Fix any $T \in (\U(n)\times \{0,1\})^m$ and $x\in \U(n)$ such that $x \notin \overline{T} $. 
    We do not need to consider any $T$ with $\overline{T}  \not\subseteq \U(n) $ since they have zero probability and do not contribute to Equation~\eqref{eq:error_uniform_over_uniforms}. 
    
    Since $|\overline{T}| =m \leq n^{1-\beta}$, we know from property~\ref{property:balanced_container} of $\H(n)$ in Lemma~\ref{lemma:labeling_set_system} that for any labeling $b \in \{0,1\}^{\supp(\overline{T}_x)}$,
     \[(1-n^{-\beta}) \frac{|\S_{\supp(\overline{T}_x)}|}{2^{|\supp(\overline{T}_x)|}}  \leq |\{h_S \in \H(n): S \in \S_{\supp(\overline{T}_x)}, h_{S}|\supp(\overline{T}_x) = b\}| \leq (1+n^{-\beta}) \frac{|\S_{\supp(\overline{T}_x)}|}{2^{|\supp(\overline{T}_x)|}}.
     \]
    Moreover, it is obvious that for any $S,S'$, we have $h_{S}|\supp(\overline{T}_x)= h_{S'}{|\supp(\overline{T}_x)}$ if and only if $h_{S}{|\overline{T}_x} = h_{S'}{|\overline{T}_x}$ as $\supp(\overline{T})$ is the set of unique elements in $\overline{T}$. Therefore, for any labeling $\ell(x)$ of $x$, there exists a unique labeling $b\in\{0,1\}^{\supp(\overline{T}_x)}$ such that $h_{S}{|\overline{T}_x} = (T, (x,\ell(x)))$ if and only if  $h_{S}{|\supp(\overline{T}_x)} = b$. This combined with the above equation implies that for any $T$ and $x\notin T$,
   the fraction of sets $S$ that contain $\supp(\overline{T}_x)$, are consistent with the labeling of $T$, and label $x$ with $1$ is close to the fraction that label $x$ as $0$. Formally, we have
    \[
    \min_{y\in\{0,1\}} \frac{\left|\{h_S \in \H(n): S \in \S_{\supp(\overline{T}_x)}, h_{S}{|\overline{T}} = T, h_S(x) = y\}\right|}{\left|\{h_S \in \H(n): S \in \S_{\supp(\overline{T}_x)}, h_{S}{|\overline{T}} = T \}\right|} \geq \frac{1}{2}\cdot\frac{1-n^{-\beta}}{1+n^{-\beta}}.
    \]
    Define by $\S[(T,x)] = \{S \in \S(n): S \in \S_{\supp(\overline{T}_x)},  h_{S}{|\overline{T}} = T\}$ the collection of sets that contain $\supp(\overline{T}_x)$ and their labeling function is consistent with the labels of $T$. The above implies that for any learner $\A$ and for any set $T$ and $x\notin T$ we have 
    \begin{equation}\label{eq:miny}
    \frac{1}{|\S[(T,x)]|}\sum_{S \in \S[(T,x)]}L(\A(T), (x,h_S(x))) \geq  \frac{1}{2}\cdot \frac{1-n^{-\beta}}{1+n^{-\beta}}.
    \end{equation}
    Note that for any $S\in \S(n)$, the multiset $T$ and $x \notin T$ have non-zero probability under $\D_S$ as training and test samples if and only if $\supp(\overline{T}_x) \subseteq S$ (i.e., $S \in \S_{\supp(\overline{T}_x)}$), and $h_{S}{|\overline{T}} = T$. Therefore, for any $T$ and $x\notin T$, the collection $\S[(T,x)]$ are exactly all the sets $S \in \S(n)$ for which $T$ and $x$ have non-zero probability as training and test samples from $\D_S$. In other words,
    each set $T$ and test point $x\notin T$ appear as a summand in Equation~\eqref{eq:error_uniform_over_uniforms} exactly $|\S[(T,x)]|$ many times. 
   Therefore, we can continue writing 
\begin{equation}\label{eq:error_partition}
\begin{aligned}
& \frac{1}{|\S(n)|}\sum_{S \in \S(n)}\left[\frac{1}{n^m}\sum_{\overline{T} \in S^m}\left[\frac{1}{n}\sum_{\substack{x\in S}}L(\A(T),(x,h_S(x)))\right]\right]\\
\\
&  = \frac{1}{|\S(n)|}\cdot \frac{1}{n^{m+1}} \sum_{\substack{T \in (\U(n) \times \{0,1\})^m \\ x \in \U(n)}}\sum_{S \in \S[(T,x)]}L(\A(T),(x,h_S(x))).
\end{aligned}
\end{equation}
Now note that from Equation~\eqref{eq:miny} we get that
\begin{equation}\label{eq:error_on_half}
\begin{aligned}
    &\sum_{\substack{T \in (\U(n) \times \{0,1\})^m \\ x \in \U(n)}}\sum_{S \in \S[(T,x)]}L(\A(T),(x,h_S(x))) \\
    &  =  \sum_{\substack{T \in (\U(n) \times \{0,1\})^m \\ x \in \U(n),\,x\notin \overline{T}}}\sum_{S \in \S[(T,x)]}L(\A(T),(x,h_S(x))) + \sum_{\substack{T \in (\U(n) \times \{0,1\})^m \\ x \in \U(n),\,x\in \overline{T}}}\sum_{S \in \S[(T,x)]}L(\A(T),(x,h_S(x)))\\
    &  \geq \sum_{\substack{T \in (\U(n) \times \{0,1\})^m \\ x \in \U(n),\,x\notin \overline{T}}}\frac{1}{2}\cdot \frac{1-n^{-\beta}}{1+n^{-\beta}}\cdot |\S[(T,x)]|.
    \end{aligned}
\end{equation}
Moreover, for any fixed $S \in \S(n)$, and any $T$ with $\supp(\overline{T})\subset S$ there are at most $|T|$ many $x\in \U(n)$ such that $x \in S$ and $x\in \overline{T}$. In other words, there are at least $n-|T|\geq n-m$ many $x$ such that $x\in S$ but $x\notin \overline{T}$. Therefore, 
\[
\forall S\in \S(n),\, \sum_{\substack{T \in (\U(n) \times \{0,1\})^m \\ x \in \U(n),\,x\notin \overline{T}}}\id [S\in \S[(T,x)]] \geq \frac{n-m}{n}\sum_{\substack{T \in (\U(n) \times \{0,1\})^m \\ x \in \U(n)}}\id [S\in \S[(T,x)]]
\]
Observe that this further means
\begin{equation}\label{eq:size_good_bad}
   \sum_{\substack{T \in (\U(n) \times \{0,1\})^m \\ x \in \U(n),\,x\notin \overline{T}}} |\S[(T,x)]| \geq \frac{n-m}{n}\sum_{\substack{T \in (\U(n) \times \{0,1\})^m \\ x \in \U(n)}}|\S[(T,x)]|\\ 
\end{equation}

Taking Equations~\eqref{eq:error_on_half} and \eqref{eq:size_good_bad} into account, we can continue Equation~\eqref{eq:error_partition} to write
\[
\begin{aligned}
    & \frac{1}{|\S(n)|}\cdot \frac{1}{n^{m+1}} \sum_{\substack{T \in (\U(n) \times \{0,1\})^m \\ x \in \U(n)}}\sum_{S \in \S[(T,x)]}L(\A(T),(x,h_S(x)))\\
    &\geq \frac{1}{|\S(n)|}\cdot \frac{1}{n^{m+1}}  \cdot \frac{1}{2}\cdot \frac{1-n^{-\beta}}{1+n^{-\beta}}\cdot\sum_{\substack{T \in (\U(n) \times \{0,1\})^m \\ x \in \U(n),\,x\notin \overline{T}}}|\S[(T,x)]| \\
     &\geq \frac{1}{|\S(n)|}\cdot \frac{1}{n^{m+1}}  \cdot \frac{1}{2}\cdot \frac{1-n^{-\beta}}{1+n^{-\beta}}\cdot\frac{n-m}{n}\cdot \sum_{\substack{T \in (\U(n) \times \{0,1\})^m \\ x \in \U(n)}}|\S[(T,x)]| 
\end{aligned}
\]
Combining the above with the fact that $\sum_{T,x}|\S[(T,x)]| = \sum_{S \in \S(n)}\sum_{T,x}\id [S \in \S[(T,x)]] = |\S(n)|\cdot n^{m+1}$, we get 
\[
\begin{aligned}
    \frac{1}{|\S(n)|}\cdot \frac{1}{n^{m+1}}  \cdot \frac{1}{2}\cdot \frac{1-n^{-\beta}}{1+n^{-\beta}}\cdot\frac{n-m}{n}\cdot \sum_{\substack{T \in (\U(n) \times \{0,1\})^m \\ x \in \U(n)}}|\S[(T,x)]| 
      & \geq \frac{1}{2} \cdot \left(\frac{n-n^{1-\beta}}{n}\right)\cdot \frac{1-n^{-\beta}}{1+n^{-\beta}}\\
       & = \frac{1}{2} \cdot \left(1-n^{-\beta}\right) \cdot \frac{1-n^{-\beta}}{1+n^{-\beta}}\\
       & \geq \frac{1}{2} - 2n^{-\beta}.
\end{aligned}
\]
where we used the fact that $m \leq n^{1-\beta}$. This proves that for any fixed deterministic learner and sample size $m\leq n^{1-\beta}$, the expectation over $\tilde{\D}_{S} \sim \DD_{n}$ of the error of $\A$ on $\tilde{\D}_{S}$ when trained on samples of size $m$ is more than $1/2 - 2n^{-\beta}$. This is enough to show that for any (randomized) learner $\A$ and sample size $m\leq n^{1-\beta}$, there exists a set $S^*_{n,m}\in \S(n)$ with distribution $\D_{S^*_{n,m}}$ such that $\A$ has error at least $1/2-2n^{-\beta}$ on $\D_{S^*_{n,m}}$ given samples of size $m$.
\end{proof}

\end{document}